\documentclass[nohyperref]{article}

\usepackage{microtype}
\usepackage{graphicx}
\usepackage{subfigure}
\usepackage{booktabs} 

\usepackage{hyperref}

\usepackage[accepted]{icml2023}

\usepackage{amsmath}
\usepackage{amssymb}
\usepackage{mathtools}
\usepackage{amsthm}

\usepackage[capitalize,noabbrev]{cleveref}

\theoremstyle{plain}
\newtheorem{theorem}{Theorem}[section]
\newtheorem{proposition}[theorem]{Proposition}
\newtheorem{lemma}[theorem]{Lemma}

\theoremstyle{definition}

\theoremstyle{remark}

\usepackage{enumerate}
\usepackage{enumitem}
\usepackage{blindtext}
\usepackage{bbm}
\usepackage{color,soul}
\usepackage{float}
\usepackage{mathbbol}
\usepackage{mathdots}
\usepackage{xcolor}
\usepackage{blindtext}
\usepackage{float}
\usepackage{verbatim}
\usepackage{amsmath}
\usepackage{amssymb}
\usepackage{mathtools}
\usepackage{amsthm}

\usepackage[textsize=tiny]{todonotes}

\newcommand{\N}{\mathbb{N}}
\newcommand{\R}{\mathbb{R}}

\newcommand{\VV}{\mathbb{V}}

\newcommand{\EE}{\mathbb{E}}
\newcommand{\Ecal}{\mathcal{E}}
\newcommand{\F}{\mathcal{F}}

\newcommand{\Pbb}{\mathbb{P}}
\newcommand{\PP}{\mathbb{P}}

\DeclareMathOperator{\argmax}{argmax}
\newcommand{\one}{{\mathbb{1}}}

\icmltitlerunning{On the convergence of the MLE as an estimator of the learning rate in the Exp3 algorithm}

\begin{document}

\twocolumn[
\icmltitle{On the convergence of the MLE as an estimator of the learning rate in the Exp3 algorithm}



\icmlsetsymbol{equal}{*}

\begin{icmlauthorlist}
\icmlauthor{Julien Aubert}{yyy}
\icmlauthor{Luc Lehéricy}{yyy}
\icmlauthor{Patricia Reynaud-Bouret}{yyy}
\end{icmlauthorlist}

\icmlaffiliation{yyy}{Universit\'e C\^ote d'Azur, CNRS, LJAD, France}

\icmlcorrespondingauthor{Julien Aubert}{julien.aubert@univ-cotedazur.fr}
\icmlcorrespondingauthor{Luc Lehéricy}{luc.lehericy@univ-cotedazur.fr}
\icmlcorrespondingauthor{Patricia Reynaud-Bouret}{Patricia.reynau-bouret@univ-cotedazur.fr}

\icmlkeywords{statistical estimation, maximum likelihood estimator, adversarial multi armed bandits, Exp3, cognition}

\vskip 0.3in
]




\printAffiliationsAndNotice{}

\begin{abstract}
  When fitting the learning data of an individual to algorithm-like learning models, the observations are so dependent and non-stationary that one may wonder what the classical Maximum Likelihood Estimator (MLE) could do, even if it is the usual tool applied to experimental cognition. Our objective in this work is to show that the estimation of the learning rate cannot be efficient if the learning rate is constant in the classical Exp3 (Exponential weights for Exploration and Exploitation) algorithm. Secondly, we show that if the learning rate decreases polynomially with the sample size, then the prediction error and in some cases the estimation error of the MLE satisfy bounds in probability that decrease at a polynomial rate. 

\end{abstract}

\section{Introduction}

\subsection{Context}

Imagine that you observe a rat in a maze, learning progressively to find food. How would you guess the learning process (or the algorithm) it actually uses ? This question is of paramount importance in cognitive science where the problem is not to find the fastest or best learning algorithm to learn a specific task but to discover the most realistic learning model (always formulated as an algorithm) \cite{Botvinick-hierarchi}.

Many learning algorithms, and in particular all those associated with reinforcement learning, are often used to model real human or animal behavior \cite{td-sutton-original, schultz1997neural}. Realistic models may take into account attentional effects \cite{GluckBower1988b}, differences in reasoning \cite{Mezzadri2021holdout}, or granularity of actions \cite{botvinick2008}. Because all these realistic ingredients vary from one individual to another, it is clear that the fitting of a particular model should be done individual by individual and experience by experience \cite{estes1956problem}.

The problem of proving that a certain model (or algorithm) is better suited to model reality than others is so crucial in cognitive science that the methodology for fitting any kind of learning algorithm to real learning data has been well established and emphasized \cite{wilson2019ten}. Any scientist wishing to develop their own new learning model can follow the same numerical experiments to test whether their model is realistic or not.

The first step of the methodology is performing MLE on the data for parameters estimation of a model \cite{daw2011trial,wilson2019ten}. Recall that we are observing an individual learning a specific task. Therefore, not only the training data ($i.e.$ the observations) strongly depend on each other, but they are also often non stationary (otherwise the individual could not have learned). These data make the study of the MLE very complex from a theoretical point of view.

This is also why extensive simulations are required by \cite{wilson2019ten}: depending on the set of chosen parameters, not only can a model learn or not learn, but there is also often a set of parameters for which the estimator behaves poorly. For instance in \cite{Mezzadri2020}'s PhD thesis, large learning rates imply a too fast learning (if the model learns at all), which in turn prevents the MLE from performing well. Unfortunately, there is a lack of theoretical guarantees on whether it is possible to estimate the parameters of these models consistently. Our goal is to prove rigorously what can be said about the properties of the MLE when fitting a learning algorithm to real data.

\subsection{Why bandits algorithms ?}

This study strongly depends on the algorithm and the experiment. One typical example is the Skinner box experiment, introduced in the 30's, whose goal was to study rodents ability to undergo operant conditioning. Once inside the box, the rat could pull one or more levers and get a reward or a punishment for it. This simplified framework allowed scientists to work in a fully controlled environment. The Skinner box paradigm is still being discussed nowadays \cite{Kang2021}.

Instead of studying a particular cognitive model and in order to work within an established general theoretical framework, we focus on the adversarial multi armed bandit problem. The algorithm we specifically study (Exp3: Exponential weights for Exploration and Exploitation) is probably the simplest algorithm for adversarial bandits \cite{lattimore-szepesvari-2020}. Even though it is not used in the cognition literature, it describes well the learning phenomena that occur in the Skinner box, and more generally all stimulus-action associations tasks (e.g. maze T-shaped, see introduction in \cite{lattimore-szepesvari-2020}). It also shares many features with famous cognitive algorithms such as component cue \cite{GluckBower1988b}: credit or loss updating, softmax transformation, etc. Finally, Exp3 has 
given rise to many variants with a variety of applications (e.g. Exp3-IX, Exp4 in \cite{lattimore-szepesvari-2020}) in more complex settings. 

Our purpose here is not to prove any new performance result of Exp3 by minimizing the cumulative loss or regret. Instead, we observe an individual
performing a learning task. We assume that it is following the Exp3 algorithm, and ask whether it is possible to infer the parameter of this specific Exp3 instance (that is, its learning rate) from the choices made by the individual. To our knowledge, the present work is the first to tackle this problem.
It is why we cannot compare our results to the classical Reinforcement Learning literature: in ML, the classical purpose is to develop effective algorithms, whereas we propose to address the question: "how to estimate parameters of a ML algorithm based on the output of an individual using this algorithm?"

\subsection{Contributions}

In Section \ref{Tetration}, we prove in a particular case that trying to estimate constant learning rates leads to poor estimation whatever the estimation procedure: the estimation error decreases more slowly than logarithmically with the number of observations.

In Section \ref{decreasinglearningratesection}, in the setting where the learning rate decreases polynomially with the number of observations, we show a polynomial decrease of the prediction error and in a particular case (Section \ref{specialcasesection}) on the estimation error on a truncated MLE.

\section{Model and notation}

\subsection{Notation}\label{Notation}

The model of experiment on which the MLE is going to be fitted is as follows. For $n$ successive iterations, the subject has to draw an arm among $K$ possible arms. Let $I_1^n=(I_t)_{1\leq t\leq n}$ be the sequence of observed arms and $(\mathcal{F}_t)_{t \leq n} := (\sigma (\{I_s, \; s \leq t \}))_{t \in n}$ the corresponding filtration.
    
As in classical cognitive experiment, the losses (or penalties) of the arms denoted by $\pi = (\pi_1, \ldots, \pi_K)$  are constant through time. These losses are bounded between $0$ and $1$ and without loss of generality, we assume that 
$1 \geq \pi_1 \geq \pi_2 \geq \cdots \geq \pi_K \geq 0.$

\subsection{Learning model: the Exp3 Algorithm}

We are going to assume that the subject picks an arm according to the following algorithm, which in this sense becomes a learning model.

\begin{algorithm}
\caption{Exp3 (Exponential weights for Exploration and Exploitation for losses)}\label{alg:cap}

Let $\eta$ be a positive real number, called the learning rate. 

Let $p_1^\eta$ be the uniform distribution over ${1, \ldots, K}$.\\
For each $t = 1, \ldots, n$,
    
    \textbullet\ Draw an arm $I_t$ from the probability distribution $p^\eta_t$.
     
     \textbullet\  For each arm $i=1,\ldots, K$, compute the estimated loss $\Tilde{\pi}_{i,t}=\displaystyle \frac{\pi_i}{p^\eta_{i,t}} \one_{I_t=i}$ and update the estimated cumulative loss $\Tilde{L}_{i,t+1}=\Tilde{L}_{i,t}+\Tilde{\pi}_{i,t}$.
    
     \textbullet\  Update the  probability distribution of picking a given arm
    $p_{t+1}^\eta=(p_{1,t+1}^\eta,\ldots,p_{K,t+1}^\eta)$, where for all $i \in \{1, \ldots, K\}$
    \begin{equation}\label{MAJ}
        p_{i,t+1}^\eta= \frac{\exp(-\eta \Tilde{L}_{i,t+1})}{\sum_{k=1}^{K}\exp(-\eta \Tilde{L}_{k,t+1})}.
    \end{equation}

\end{algorithm}

In the most general case, Exp3 is able to cope with losses that are depending on time and it has been proved that it achieves sublinear bounds for the convergence of the pseudo-regret  \cite{DBLP:journals/corr/abs-1204-5721}.

However, for the purpose of the present work, we have simplified the set-up. Indeed, even if the algorithm is able to cope with time-varying losses, we observe it in a rigid framework that is planned by the cognitive experiment itself, for which the most classic framework correspond to fixed losses. Indeed, in cognitive science, one often evaluates very robust and realistic learning processes in toy situations where most of the variability is cancelled.

The parameter $\eta$ in Exp3 can originally be also chosen has a time varying parameter. Here we have decided that this quantity is fixed during this experiment, because this is the parameter that the MLE is estimating based on the $n$ observations.

If the pseudo-regret bound obtained in \cite{DBLP:journals/corr/abs-1204-5721} is optimal for $\eta$ of the order of $1/\sqrt{n}$, there are situations where the algorithm can learn even if $\eta$ is constant (see Section 3) or decreasing at a different rate. Therefore from a statistical point of view, it is not clear that the subject uses a constant $\eta$ or a parameter $\eta$ that tends to $0$ with $n$. Let us go even further: the fact that the algorithm learns or not, should not be an absolute criterion, since we might have subjects that are unable to learn even after many iterations and might just give up.
This is the main difference with other works on Exp3. Here we take Exp3 as a realistic model that can be fitted to data : the range of $\eta$ that matters to us is the one for which we can guarantee a good quality of estimation of the learning rate $\eta$.

In the sequel, we denote by $\PP^\eta$ the probability when the learning parameter is $\eta$, and $\EE_{\eta}[.]$ is the associated expectation. Finally, if $f$ is a function, $f^{(i)}$ denotes its $i$-th iteration, and $f^{-1}$ denotes its inverse when it exists. Also $O_n(1)$ is sequence that is bounded when $n$ tends to infinity.

\section{Tetration behaviour}\label{Tetration}

In this section, we illustrate the poor performance of the MLE when $\eta$ is constant. We focus on a particular case where there are only two arms and $\pi_2=0$. 
Studying $KL( \PP_{I_1^n}^\eta  || \PP_{I_1^n}^\delta)$, the Kullback-Leibler (KL) divergence between $\mathbb{P}^\eta_{I_1^n}$ and $\mathbb{P}^\delta_{I_1^n}$ (the distributions of the vector of pulled arms $I_1^n$ with the learning rates $\eta$ and $\delta$ respectively), we show that some parameters do not separate well whatever the statistical procedure we use.

\subsection{A particular setup}

It is sufficient with only two arms to look only at the time steps at which the worst arm (arm 1) is pulled, that is 
$T_0=0$ and for all $i \geq 0$
, $T_{i+1}=\inf \{t \geq T_i +1 \text{ , } I_t = 1\}.$

 Indeed, only at that times, is the probability changing: with 
 $q_0^\eta=\frac{1}{2}$ and for all $i \geq 0$,
\begin{equation*}
    q_{i+1}^\eta = \frac{q_i^{\eta} e^{-\pi_1 \eta/q_i^{\eta}}}{(1-q_i^\eta)+q_i^{\eta} e^{-\pi_1 \eta/q_i^{\eta}}}.
\end{equation*}
Note that $p_{1,t}^\eta$ is simply $p_{1,t}^\eta = \displaystyle \sum_{i \geq 0} q_{i}^\eta \one_{T_{i}+1 \leq t \leq T_{i+1}}.$

We show in Lemma \ref{summarypropertiesqi} from the Appendix that the $q_i^\eta$'s are decreasing and tend to $0$, 
and
that the increments $T_{i+1}-T_i$ are independent and geometrically distributed with parameter $q_i^\eta$.

Therefore when $i$ increases, the distance between the $T_i$ increases, making it more and more difficult to observe an error 
(the arm 2 is pulled). This is what we quantify in the next paragraph.

\subsection{Behavior of $q_i^\eta$}

 
 Define $I(\eta)$ and $J(n,\eta)$ as follows.

\begin{equation*}
\left \{ \begin{array}{ll}
     I(\eta) & := \max \{i \in \N \text{ , } q_i^\eta \geq \eta \pi_1 \}  \\
     J(n,\eta) & := \max \{i \in \N \text{ , } q_i^\eta \geq \frac{1}{n} \}
\end{array}
\right.
\end{equation*}
Note that when $\eta$ is constant, so is $I(\eta)$.

The following proposition shows how fast $q_i^\eta$ decreases for $i\geq I(\eta)$.

\begin{proposition}\label{afterIeta}
Let $f$ be the function defined as $f(x) = e^{\frac{x}{2}}$. Then, for all $k \geq 0$,
\begin{equation*}
    q_{I(\eta)+k+1}^\eta \leq \frac{1}{f^{(k)}(2)} ,
\end{equation*}
with the convention $f^{(0)} = \text{Id}$.
In particular, $J(n,\eta) \leq I(\eta) + \log^*(n) + 1$, where
\begin{equation*}
    \log^*(n) = \max \{ k \in \N \text{ , \; } (f^{-1})^{(k)}(n) \geq 2 \}.
\end{equation*}
\end{proposition}

\begin{proof}
See section \ref{proofoftetration} of the Appendix.
\end{proof}

This result shows that, as soon as $i\geq I(\eta)$, $q_i^\eta$ decreases extremely rapidly down to $0$, as fast as a tetration (that is an iterated exponentiation) (see \cite{Tetration}). This means reciprocally that the number of indices necessary to pass from $I(\eta)$ to $J(n,\eta)$ is in essence bounded. For instance, with $n=10^{23}$, $\log^*(n) \leq 6$: even with an unrealistic number of observations, it means that after $I(\eta)+7$ errors, which is a constant and quite small number, $q_i^\eta$ becomes as small as $10^{-23}$. This essentially means that this probability becomes null, up to a reasonable computer precision.

Said differently (still with $n=10^{23}$), the statistical problem is almost equivalent (up to computer precision) to the one consisting in observing a fixed number $m$ of independent geometric variables (with $m \leq I(\eta)+7$).



\subsection{Bounded Kullback-Leibler divergence} 

Let $\tau(n)=\max \{k \in \N \text{ , } T_k \leq n \}$. 
Then,
\begin{equation}\label{expressionoftheKLdivergence}
\begin{split}
KL(& \PP_{I_1^n}^\eta || \PP_{I_1^n}^\delta) = \EE_{\eta} \bigg[ \sum_{i=0}^{\tau(n)-1} ((T_{i+1}-T_i-1)   \log \frac{1-q_i^\eta}{1-q_i^\delta} \\& +\log \frac{q_i^\eta}{q_i^\delta})\bigg]  + \EE_{\eta} \bigg[ (n-T_{\tau(n)})\log \frac{1-q_{\tau(n)}^\eta}{1-q_{\tau(n)}^\delta} \bigg]
\end{split}
\end{equation}

The previous proposition leads to the following result.
\begin{theorem}[Existence of parameters that are hard to distinguish]\label{boundingklconstant} 
Let $R$ such that $0 < R\pi_1 < 1$ and 
let $\beta>0$. There exists an integer $n_0$ depending on $R$, $\pi_1$ and $\beta$, a constant $\Delta > 0$ depending on $\pi_1$ and $\beta$ and a constant $c>0$ depending on $\pi_1$ such that for all $n\geq n_0$, there exists
\begin{itemize}
\item $\delta \in [0,R]$ such that $\delta\geq \displaystyle \frac{c}{\log^*(n)}$,
\item $\eta \in [0,R]$ such that $\eta =  \displaystyle \delta+\frac{1}{(\log(n))^{1+\beta}}$
\end{itemize}
such that
\begin{equation}\label{encadrementKL}
 KL(\PP_{I_1^n}^\eta || \PP_{I_1^n}^\delta) \leq \Delta.
\end{equation}


\end{theorem}
\begin{proof}
See Section \ref{proofoftheoremetaconstant} of the Appendix.
\end{proof}

This theorem  becomes really interesting in the light of Theorem 2.2. of \cite{Tsybakov} that tells us that if one finds two distributions such that $KL(\PP_{I_1^n}^\eta || \PP_{I_1^n}^\delta) \leq \Delta$, then any statistical procedure able to distinguish $\eta$ from $\delta$ is doomed to make an error in probability of at least $1/4\exp(-\Delta)$.

It means that whatever the method, one cannot distinguish these two parameters without an almost fixed probability of error. This implies that any estimation procedure of $\eta$ cannot converge at a faster rate than the distance between $\eta$ and $\delta$, that is a logarithmic rate in $O_n(\log(n)^{-(1+\beta)})$. 

Note that due to technicalities, $\delta$ may have to tend to 0 but at a slower rate than $1/\log^*(n)$, which as said previously is almost constant. So in practice this corresponds to the case where the learning parameters are constants. 

We think the lower logarithmic bound for $|\eta-\delta|$ is very pessimistic and that in fact it is quite likely that the phenomenon appear even for $|\eta-\delta| =O_n((\log^*(n))^{-1})$, maybe to some power, but we have not been able to prove it. In any case, this rate is much slower than the polynomial rate we obtain in the next two sections for a decreasing learning rate and the truncated MLE.  Moreover, simulations (see Section 6) at least confirm that for a fixed error (say 5\%), the estimation error of the MLE  is not decreasing as a function of $n$ for a constant learning rate.

\section{Decreasing learning rate}\label{decreasinglearningratesection}

Let us now turn to the case where the parameter $\eta=\eta_n$ decreases with $n$, the number of observations. We consider the general case of $K$ arms introduced in Section \ref{Notation}. 

Allowing $\eta_n$ to decrease with $n$ is allowing $I(\eta_n)$ to grow with $n$. Proposition \ref{afterIeta} intuitively shows that the observations that matter for the estimation are those obtained for $t\leq I(\eta_n)$. After $I(\eta_n)$, the probability of pulling the worst arm is so small that it is negligible from a numerical point of view and almost uninformative from an estimation point of view (see Section \ref{Tetration}).

Truncating the observations has therefore a twofold interest: not only are we deleting uninformative data, but we are also removing the values in the log likelihood which could explode at the speed of a tetration and lead to numerical issues (see Section \ref{Numericalillustrations} for an illustration).

The question now is where to truncate the observations. From an estimation point of view, it is not possible to choose $I(\eta_n)$ as a truncating value since it depends on the unknown parameter $\eta_n$. Instead we introduce a parameter $\varepsilon > 0$, and given the number of observations $n$, we want to find a stopping term $\Upsilon_n \in \N$ guaranteeing that $p_{k,t}^{\delta_n}$, the probability of choosing an arm $k$ at round $t$, remains greater than $\varepsilon$ for all $t\leq\Upsilon_n$, whatever the choice of $\delta_n$ in the set of possible parameters.


\subsection{The new set-up and the truncated log-likelihood}

In the general setting, recall that $K$ is the number of arms and $\pi$ the sequence of losses.
From now on, for some $\alpha$ fixed and known in  $(0,1)$, let $\eta_n \displaystyle = \frac{\eta_0}{n^\alpha \pi_1}$.
We assume that the unknown parameter $\eta_0$ belongs to $
    \Theta = [r, R]$
with $0<r<R$ two known positive constants.




Let $0 < \varepsilon < \frac{1}{K}$ be a fixed known threshold, corresponding roughly in practice to the numerical precision. 

The following result gives an absolute bound $\Upsilon_n$  that guarantees that the $p_{k,t}^{\delta_n}$'s are always larger than the threshold~$\varepsilon$.

\begin{proposition}
\label{choiceofTau}
Let
\begin{equation}\label{Tn}
\Upsilon_n = \Big \lfloor (\frac{1}{K}- \varepsilon) \frac{n^\alpha}{R}\Big \rfloor.
\end{equation}
Then
\begin{equation*}
\forall \delta_0 \in \Theta, \  \forall t \leq \Upsilon_n, \  \forall 1 \leq k \leq K, \quad
    p_{k,t}^{\delta_n} \geq \varepsilon,
\end{equation*}
with $\delta_n=\displaystyle \frac{\delta_0}{n^\alpha \pi_1}$.
\end{proposition}

\begin{proof}
See Section~\ref{proofofpropositionchoiceofTau} of the Appendix.
\end{proof}

From now on,  we define the truncated log likelihood as follows
\begin{equation}
\label{truncatedloglikelihood}
\forall \delta_0 \in \Theta \text{, \; \; }
    \ell_{n,\varepsilon}(\delta_0)
    := \sum_{t=1}^{\Upsilon_n} \sum_{k=1}^K \log(p_{k,t}^{\delta_n}) \one_{I_t = k},
\end{equation}
with $\delta_n \displaystyle=\frac{\delta_0}{n^\alpha \pi_1}$ and $\Upsilon_n$ given by \eqref{Tn}.

\subsection{Upperbound on the prediction error}
Let $\eta_0\in \Theta$ be the true parameter, meaning Exp3 is used by the subject with a learning parameter $ \displaystyle \eta_n=\frac{\eta_0}{n^\alpha \pi_1}$.  Let \begin{equation*}\widehat{\eta}_0 = \argmax_{\delta_0 \in \Theta} \ell_{n,\varepsilon}(\delta_0).\end{equation*} The estimator of the learning rate of Exp3 is therefore $ \displaystyle \widehat{\eta}_n= \frac{\widehat{\eta}_0}{n^\alpha \pi_1}$. The following result shows that the prediction error converges to zero with~$n$.

\begin{theorem}
\label{leastsquareerror}
For any $t \geq 0$, and any $\delta_0, \delta_0' \in \Theta$, let
\begin{equation*}
\|p_t^{\delta_n} - p_t^{\delta'_n} \|_2^2 
    = \sum_{k=1}^K |p_{k,t}^{\delta_n}-p_{k,t}^{\delta'_n}|^2,
\end{equation*}
with $\delta_n=\frac{\delta_0}{n^\alpha \pi_1}$ (resp. $\delta'_n=\frac{\delta'_0}{n^\alpha \pi_1}$.)

Then, for all $n \geq (R / \varepsilon^2)^{1/\alpha}$ and $x \geq 0$, with $\Pbb_{\eta_n}$-probability at least $1-e^{-x}$,
\begin{equation*}
\frac{1}{\Upsilon_n} \sum_{t=1}^{\Upsilon_n} \| p_{t}^{\eta_n}-p_{t}^{\widehat{\eta}_n} \|_2^2
    < \frac{9 c}{\varepsilon} \left( \sqrt{\frac{x+1}{\Upsilon_n}} + \frac{x+1}{\Upsilon_n} \right),
\end{equation*}
where $c$ is given by Lemma \ref{lem_ptlipschitz} below ($c=11$ works).
\end{theorem}

\begin{proof}
See Section \ref{proofofleastsquareerror} of the Appendix for a complete proof but a sketch of the proof with the main milestones is given in the next section.
\end{proof}

Note that if we take $\varepsilon=10^{-15}$ the previous bound is rather pessimistic, but for $\varepsilon=10^{-2}$ (and $K$ much smaller than 100), we find almost identical $\Upsilon_n$ from a theoretical point of view (see Proposition \ref{choiceofTau}). Even in practice (See Section 6), one can show that the first time, $\Upsilon_{max}$, that the $p_{k,t}^{\eta_n}$'s pass below $\varepsilon$, does  not vary much  between both choices for $\varepsilon$. Since $\varepsilon$ does not impact the estimation per se (only through $\Upsilon_n$), one can opt for a rather reasonable choice of $\varepsilon$ when using this bound. The fact that passing $10^{-2}$ or passing $10^{-15}$ happen at almost the same time $t$ can be linked heuristically to the tetration phenomenon (see Proposition \ref{afterIeta} that applies whatever $\eta$ and in particular with $\eta=\eta_n$). 

This bounds also means that if one stays away from $\varepsilon$, the behavior of the truncated MLE seems slower than the parametric convergence rate (that should be in $\Upsilon_n^{-1}$), where $\Upsilon_n$ represents the number of observation used in the truncated log-likelihood. Nevertheless the rate is polynomial in $O_n(\Upsilon_n^{-1/2})=O_n(n^{-\alpha/2})$. We do not know, if this bound is tight or not, but the simulations (see Section 6) seem to confirm this rate.

\subsection{Sketch of proof}

\begin{enumerate}
\item By definition of $\widehat{\eta}_0$, we have $\ell_{n,\varepsilon}(\widehat{\eta}_0)\leq \ell_{n,\varepsilon}(\eta_0)$.
By concentrating each quantity around its compensator in the martingale sense, we get a bound on a random version of a truncated Kullback-Leibler divergence, which is in turn lower bounded by the prediction error. 

\item However, the concentration inequality is not straightforward because $\widehat{\eta}_0$ is random and depends on the same sample. Hence we need to control the deviation of a supremum of martingales. 

\item To do so, we turn to the work of \cite{https://doi.org/10.48550/arxiv.0909.1863}, which provides a bound on supremum of the type 
$Z=\sup_{\delta \in \Theta} Y_\delta$
as long as $Y_\delta-Y_{\delta'}$ has a specific Laplace transform that shows  a certain regularity  w.r.t. a distance on $\Theta$.

\item In particular, to prove this form of regularity of the Laplace transform,  the adequate distance should control a distance on the $p^{\eta_n}_{k,t}$'s.
\end{enumerate}

Let us describe now briefly how we do these steps, starting from step 4.



\subsubsection{Step 4:  Lipschitz character  of \texorpdfstring{$p_t^\delta$}{the probabilities} } 

\begin{lemma}
\label{lem_ptlipschitz}
For any $t \geq 0$ and $\delta_0, \delta'_0 \in \Theta$, let
\begin{equation*}
\|p_t^{\delta_n} - p_t^{\delta'_n} \|_{\infty}
    = \max_{1 \leq k \leq K} |p_{k,t}^{\delta_n}-p_{k,t}^{\delta'_n}|.
\end{equation*}
There exists a numerical constant $c$ (for instance $c=11$) such that for all $\delta_0, \delta_0' \in \Theta$, $n \geq (R / \varepsilon^2)^{1/\alpha}$ and $1 \leq t \leq \Upsilon_n = (\frac{1}{K} - \varepsilon) \frac{n^\alpha}{R}$,
\begin{equation}
\label{majorationlipschitzpt}
\|p_{t}^{\delta_n'}-p_{t}^{\delta_n}\|_{\infty}
    \leq c \frac{|\delta_0-\delta_0'|}{R}.
\end{equation}
\end{lemma}

\begin{proof}
See Section~\ref{proofoflem_ptlipschitz} of the Appendix.
\end{proof}

\subsubsection{Step 3: Regularity of the Laplace transform}

Let $\delta_0 \in \Theta$ and let
\begin{eqnarray*}
X_{\delta_n}& = & \ell_{n,\varepsilon}(\delta_0) \displaystyle - \sum_{t=1}^{\Upsilon_n} \sum_{k=1}^K \log(p_{k,t}^{\delta_n}) p_{k,t}^{\eta_n}\\
&=& \sum_{t=1}^{\Upsilon_n} \sum_{k=1}^K \log(p_{k,t}^{\delta_n}) \left[\one_{I_t = k}-p_{t,k}^{\eta_n}\right].
\end{eqnarray*}
Note that, since the true learning rate parameter is $\eta_n$, the quantity $X_{\delta_n}$ can be seen as a martingale stopped at time $\Upsilon_n$, whatever the value of $\delta_n$.

The following proposition is a requirement to apply the work by \cite{https://doi.org/10.48550/arxiv.0909.1863}.

\begin{proposition}
\label{hypothesebaraud}
Let $c$ be the constant defined in Lemma~\ref{lem_ptlipschitz}. For each $\delta_0,\delta_0' \in \Theta^2$, let
\begin{equation*}
d(\delta_0',\delta_0) = \frac{c}{R \varepsilon} |\delta_0-\delta_0'|.
\end{equation*}
Then, for all $(\delta_0,\delta_0') \in \Theta^2$, $n \geq (R / \varepsilon^2)^{1/\alpha}$ and $|\lambda| < \frac{1}{d(\delta_0,\delta_0')}$,
\begin{equation*}
\EE_{\eta_n}[e^{\lambda (X_{\delta_n}-
     X_{\delta_n'})}]
    \leq \exp \left( \frac{(\lambda d(\delta_0',\delta_0))^2 \Upsilon_n}{2(1- \lambda d(\delta_0',\delta_0))}\right).
\end{equation*}
\end{proposition}

\begin{proof}
See Section~\ref{proofofbaraud} of the Appendix.
\end{proof}

\subsubsection{Step 2: Resulting concentration inequality for the supremum}

Next, it is sufficient to apply the work by \cite{https://doi.org/10.48550/arxiv.0909.1863} to $Y_{\delta_0}=X_{\delta_n}-X_{\eta_n}$.

\begin{theorem}[Bernstein-type inequality]
\label{theorembaraud}
Let $c$ be the constant from Lemma~\ref{lem_ptlipschitz}. Let $\displaystyle
Z_n = \sup_{\delta_0 \in \Theta}(X_{\delta_n}-X_{\eta_n}).$
Then, for all $n \geq (R / \varepsilon^2)^{1/\alpha}$ and $x \geq 0$,
\begin{equation}
\label{theorem2.1}
\PP_{\eta_n}\bigg[Z_n \geq \frac{18 \, c}{\varepsilon} (\sqrt{\Upsilon_n} \sqrt{x+1} + x+1) \bigg] \leq e^{-x},
\end{equation}
\end{theorem}

\begin{proof}
See Section \ref{proofofbernsteintype} of the Appendix.
\end{proof}

\subsubsection{Some comments about Step 1 and the general methodology}

Step 1 is very classic. Moreover, this kind of methodology to derive non asymptotic results for Maximum Likelihood Estimators has been already used in much more general set-up, in particular by \cite{Spokoiny}. However note that despite the very general nature of Spokoiny's work, it is not straighforward to apply it here, because the derivatives of the log-likelihood with respect to $\delta$ are not straightforward, thanks  to the recursive nature of $p^{\delta_n}_{k,t}$. Therefore we have been forced to replace this kind of direct approach on the derivative, by a less direct one. In particular, our general result stops at the prediction errors and we are not able to give bounds directly on the estimation error, $|\widehat{\eta}_0-\eta_0|$,  because a general control on the derivatives of the log-likelihood is missing. We derive bounds on the estimation error only in a specific case that we detail in the next section, because in this case we can find recursive bounds on the derivatives.





\section{Estimation error in a special case}\label{specialcasesection}

We prove an upper bound in large probability on the estimation error $|\eta_0 - \widehat{\eta}_0|$, only in the case when we have only two arms $K=2$ and the loss of arm 2 is null, that is  $\pi_2 = 0$. 

\subsection{A lower bound on the prediction error}
To pass from the control of the prediction error to the control of the estimation error, we need first to find an adequate lower bound.
\begin{proposition}\label{Lowerboundleastsquare}
Let $N_0 = 0$ and $N_t = \displaystyle \sum_{s=1}^{t} \one_{I_s=1}$ be the number of times arm 1 is pulled until $t\geq 1$. 

There exists a constant $m_{\pi_1,\varepsilon} > 0$, depending only on $\pi_1$ and $\varepsilon$ such that, for all $\delta_0, \eta_0 \in \Theta$
\begin{equation}
    \sum_{t=1}^{\Upsilon_n}  |p_{1,t}^{\delta_n}- p_{1,t}^{\eta_n}|^2 \geq  m_{\pi_1,\varepsilon} |\delta_n - \eta_n|^2 \sum_{t=1}^{\Upsilon_n} N_{t-1}^2.
\end{equation}
\end{proposition}

\begin{proof}
See Section \ref{proofoflowerboundleastsquare} of the Appendix.
\end{proof}

\subsection{Upper bound for the estimation error}

Thanks to Theorem \ref{leastsquareerror} and Proposition \ref{Lowerboundleastsquare}, the following  bound on the estimation error holds with large probability.

\begin{proposition}\label{probabilityestimationerror} For $x \geq 0$ and $n \in \N$, define by
\begin{equation*}
    G_n(x) =\frac{2}{5} \displaystyle \sqrt{2 (\log 2 \Upsilon_n+x)}  \Upsilon_n^{\frac{1}{2}} +  \log 2 \Upsilon_n + x
\end{equation*}
There exists a constant $M_{\pi_1,\varepsilon}>0$, depending only on $\pi_1$ and $\varepsilon$ and a positive sequence $B_n$ verifying
\begin{equation*}
    B_n := \frac{1}{48} \Upsilon_n + O_n(1),
\end{equation*}
such that for all $n \geq (R/\varepsilon^2)^{\frac{1}{\alpha}}$  and $x \geq 0$ such that $B_n > G_n(x)$, with $\PP_{\eta_n}$-probability at least $1-2e^{-x}$,
\begin{equation*}
     |\widehat{\eta}_0 - \eta_0| \leq M_{\pi_1,\varepsilon} \left(\frac{\sqrt{(x+1)\Upsilon_n} + x+1}{B_n-G_n(x)} \right)^{\frac{1}{2}}.
\end{equation*}

\end{proposition}

\begin{proof}
See Section \ref{proofofprobabilityestimationerror} of the Appendix.
\end{proof}

Proposition~\ref{probabilityestimationerror} shows that for large $n$, $
|\eta_0 - \widehat{\eta}_0| =O(\Upsilon_n^{-\frac{1}{4}})=O(n^{-\frac{\alpha}{4}})$,
which means that the truncated MLE converges at a polynomial rate, which is slower than the classic parametric rate, but clearly faster than the logarithmic rate obtained in Section 3 for a constant learning rate. As for the prediction error, we do not know if this rate is tight, but it seems to be confirmed by the simulations (see Section \ref{Numericalillustrations}).

\section{Numerical illustrations}\label{Numericalillustrations}


\subsection{Numerical set-up}
All the simulations have been conducted with \texttt{R}. The simulation of the $n$ iterative subject's choices have been simulated according to an Exp3 algorithm with a given learning parameter $\eta$, $K$ arms and losses given by $\pi=(\pi_1,...,\pi_K)$. Note that if the $p_{k,t}^\eta$ are too small, the Exp3 simulation stops because "NA" can be returned when evaluating \eqref{MAJ}. In all the simulations that are proposed here, thanks to the choices of $n$ and $\eta$ this has never happened.

\subsection{Constant learning rate}

We choose $\eta=$ 0.3. The log-likelihood is given by
\begin{equation}
\label{loglikelihood}
    \forall \delta \in (0.1,0.8) \text{, \, } \ell_{n}(\delta)
    := \sum_{t=1}^{n} \sum_{k=1}^K \log(p_{k,t}^{\delta}) \one_{I_t = k}.
\end{equation} 

\begin{figure}[h!]
\includegraphics[width=4.2cm,height=6cm]{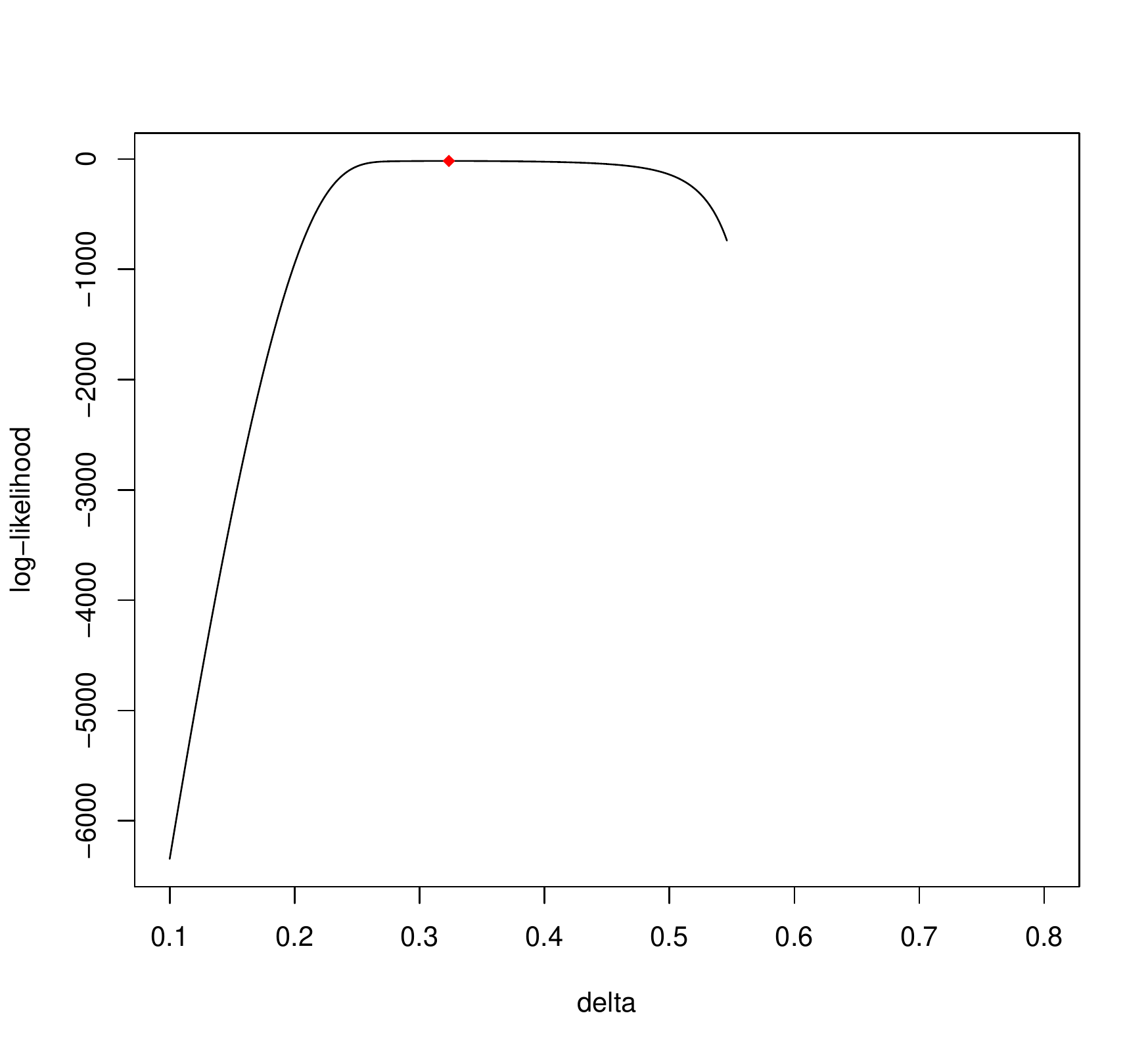}\includegraphics[width=4.2cm,height=6cm]{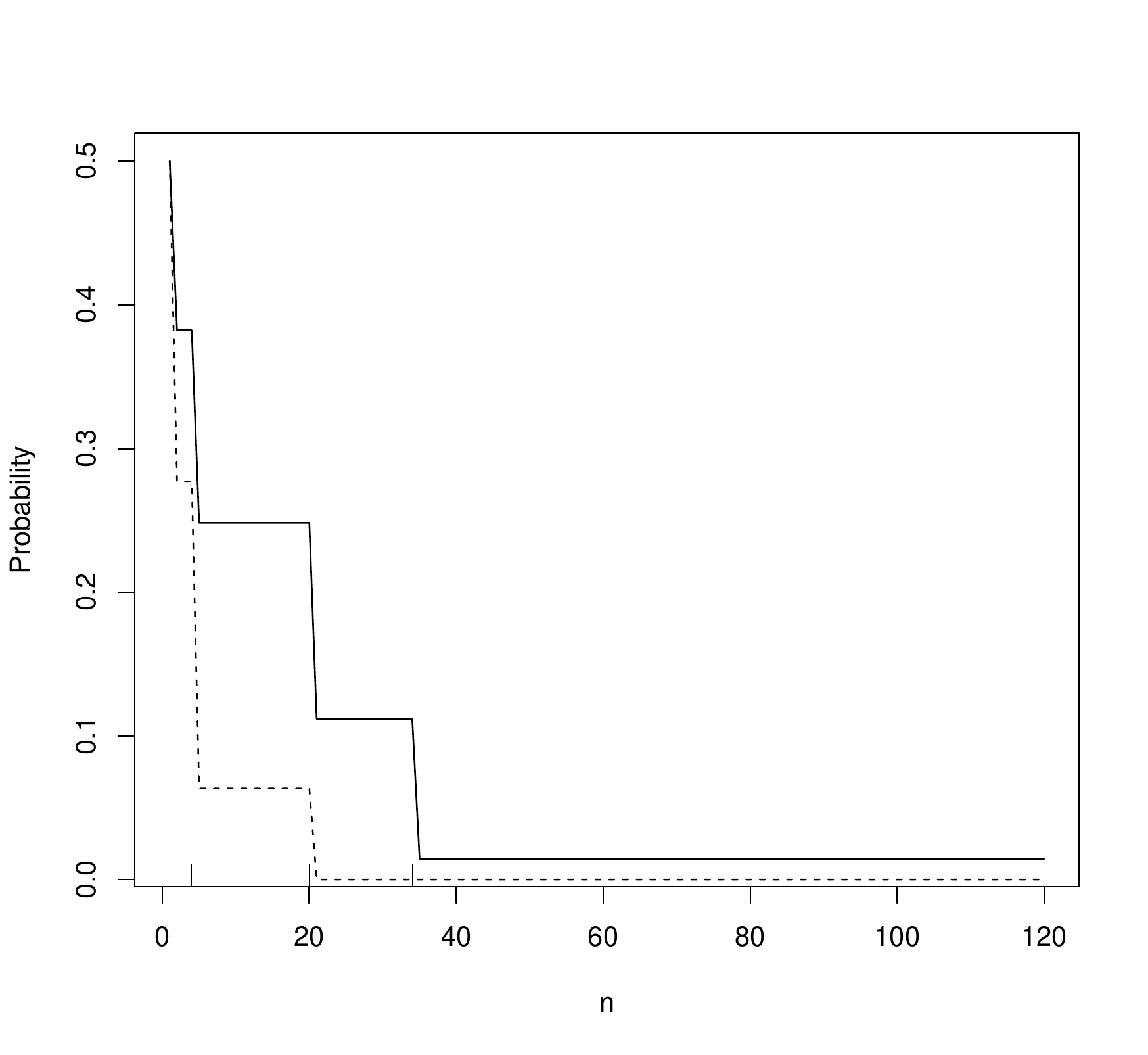}
\caption{Log-likelihood. On the left, the log-likelihood with the MLE in red. The simulation is performed via the Exp3 algorithm with $K=2$, $\pi = (0.8,0)$ and  $\eta=0.3$. The value of the log-likelihood at $\delta=0.6$ is \texttt{-Inf}. On the right, evolution of $p_{1,t}^\eta$ (plain black) and  of $p_{1,t}^\delta$ (dotted black), with  $\delta=0.6$. After $n=20$, $p_{1,t}^\delta$ is considered as null by the computer. }
\label{vrais}
\end{figure}

Figure \ref{vrais} shows the shape of the log-likelihood when $K=2$ and $\pi = (0.8,0)$ on one simulation. For some parameters $\delta$, the log-likelihood is computed as \texttt{-Inf} because some probabilities become null for the numerical precision (see the right part of Figure \ref{vrais}). Because of this repeated \texttt{-Inf} value, we have not been able to perform simulations for such large $\eta$ with more than 2 arms. 

In the situation with $K=2$, $\pi = (0.8,0)$
and $\eta=0.3$, we maximised the log-likelihood thanks to the function \texttt{DEoptim} in \texttt{R} inside the interval $(0.1,0.8)$, with the default parameters and a maxiter value equal to 50. With this choice of set-up, \texttt{DEoptim} returns a correct estimator (red point in Figure \ref{vrais}), but note that the log-likelihood is very flat around the MLE so that spurious estimation is likely.

\begin{figure}[h!]
\centerline{\includegraphics[scale = 0.4]{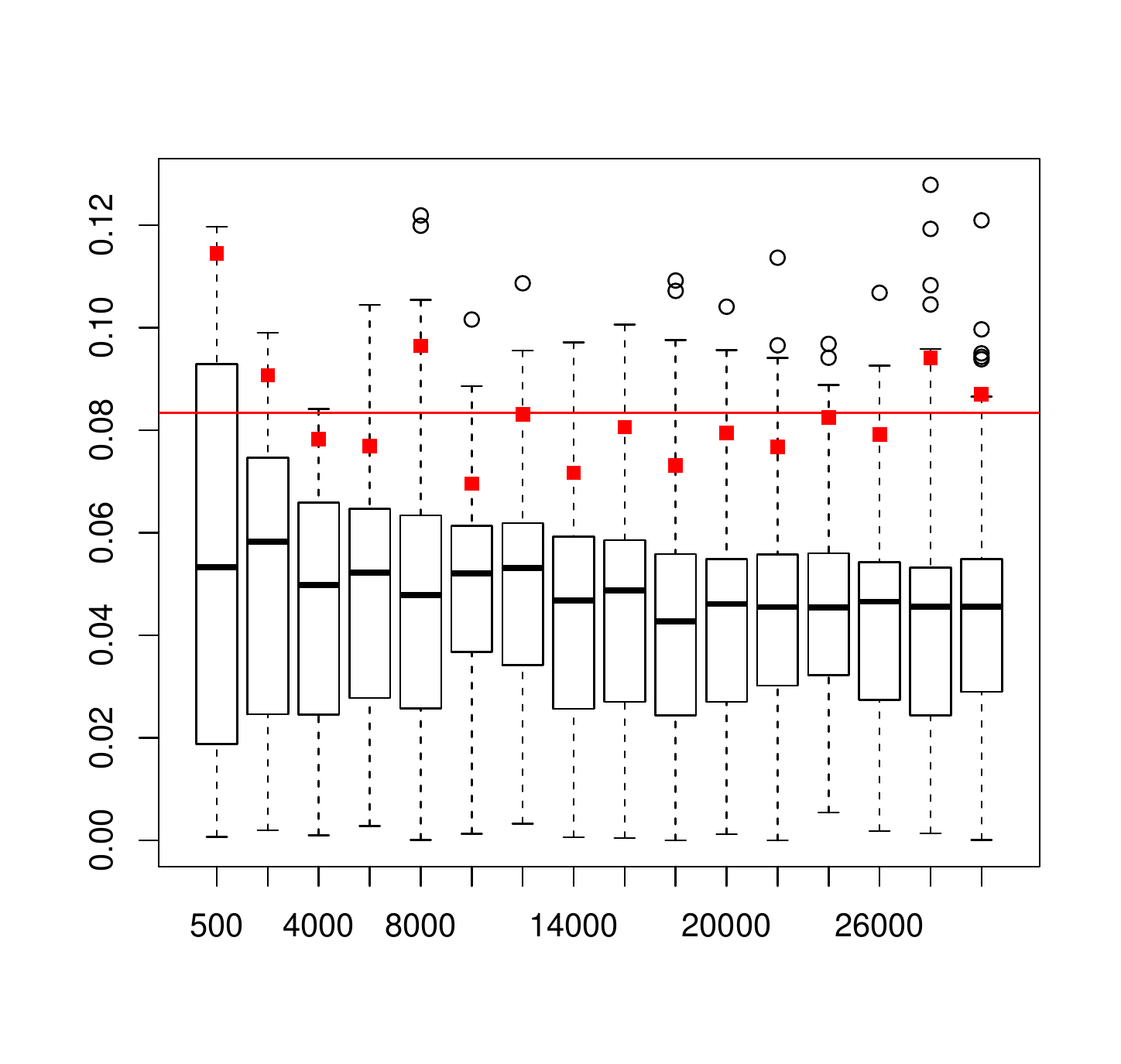}}
\caption{Estimation error when $\eta=0.3$. In the set-up of Figure \ref{vrais} boxplot over 100 simulations of $|\hat{\eta}-\eta|/\eta$. The red squares are the 95\% quantiles. The Spearman rank correlation test of decreasing character of these 95\% quantiles has a p-value of 0.39. The red horizontal line is the mean of these quantiles.}
\label{est_etacst}
\end{figure}

We performed 100 simulations of this situation for various sample size between $n=$500, and $n=30 000$ (see Figure \ref{est_etacst}). If the variance of the error decreases with $n$, the  general tendency of the estimation error does not. The Spearman rank correlation test does not detect any decreasing character of the 95\% quantile of the error as $n$ increases. This is in total adequation with Theorem \ref{boundingklconstant} which basically states there are some $\delta, \eta$ that cannot be distinguished whatever $n$. Hence the simulations corroborate that estimation cannot be done in a satisfactory manner when $\eta$ is constant and also suggest that even the logarithmic lower bound of Theorem \ref{boundingklconstant} is pessimistic.

\subsection{Decreasing learning rate and the stopping criterion}

Here we choose $\alpha=1/2$, $\eta= \displaystyle \eta_n=\frac{0.3}{\pi_1 n^\alpha}$. 

The truncated log-likelihood is defined by
\begin{equation}
\label{truncated loglikelihood}
   \forall \delta_0 \in (0.1, 0.8), \text{\, } \ell_{n}(\delta_n)
    := \sum_{t=1}^{\Upsilon_{max}} \sum_{k=1}^K \log(p_{k,t}^{\delta_n}) \one_{I_t = k}.
\end{equation}

Following the set-up of Section 4 and also to avoid values \texttt{-Inf} for the log-likelihood,  we stop the log-likelihood, for a given $\varepsilon$,
at 
\begin{equation}\label{Upsmax}
    \Upsilon_{\max} = \sup\{t \in \N : p_{k,t}^{\delta_n} > \varepsilon, \text{\, } \forall \delta_n \in \texttt{grid}_n, \forall k\},
\end{equation}
with $\texttt{grid}_n$ a regular grid of 50 points in $(\frac{0.1}{\pi_1\sqrt{n}}, \frac{0.8}{\pi_1\sqrt{n}})$ and $1 \leq k \leq K$.

Note that $\Upsilon_{max}$ is random and larger than the theoretical choice $\Upsilon_n$ of Section 4. 

For a fixed $n$ taking various values  between $n=500$ and $n=30000$, we perform 100 simulations with $K=4$ arms and $\pi=(0.8,0.6,0.4,0.2)$ and computed each time the corresponding $\Upsilon_{max}$. The average is shown on Figure \ref{evolutionofTmax}. In particular, we see that the bound    $\Upsilon_n$ in $\sqrt{n}$ in Proposition \ref{choiceofTau} is a tight bound that reflects well the behavior of the first time to cross $\varepsilon$. We also see that the value of $\varepsilon$ has almost no impact on the choice of the truncation inside the log-likelihood, since all curves are almost confounded.

\begin{figure}[h!]
\centerline{\includegraphics[scale = 0.43]{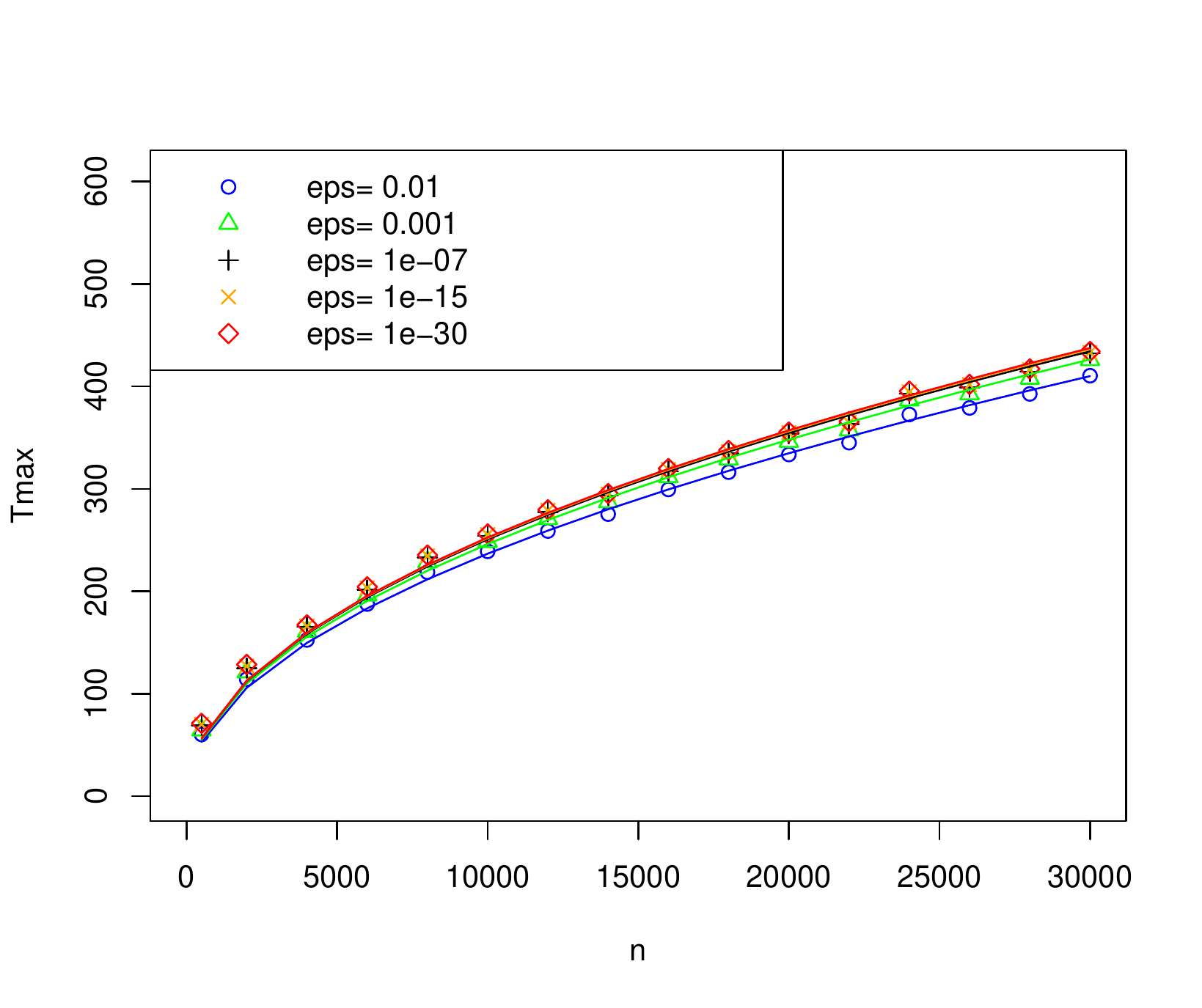}}
\caption{Evolution of $\Upsilon_{\max}$ as a function of $n$ for different $\varepsilon$. The points are the average of the $\Upsilon_{max}$ over 100 simulations with $K=4$ arms and $\pi=(0.8,0.6,0.4,0.2)$ and $\eta= \displaystyle \eta_n=0.3/(\pi_1 \sqrt{n})$. The curves for the different values of $\varepsilon$ are found by regression of the points on the square root curve.} 
\label{evolutionofTmax}
\end{figure}

\subsection{Performance with decreasing learning rate}

We take again $\eta_n = \displaystyle \frac{0.3}{\pi_1 \sqrt{n}}$. We perform 100 simulations. The sample size $n$ takes various values between 500 and 30000. We maximize the truncated log-likelihood given by \eqref{truncatedloglikelihood} truncated at $\Upsilon_{max}$ given by \eqref{Upsmax} with $\varepsilon=10^{-7}$, thanks to the \texttt{DEoptim} function of \texttt{R}. We consider two cases: $K=2$ arms with $\pi = (0.8,0)$ and $K=4$ arms with $\pi = (0.8,0.6,0.4,0.2)$. 

\begin{figure}[h!]
\includegraphics[width=4.2cm,height=7cm]{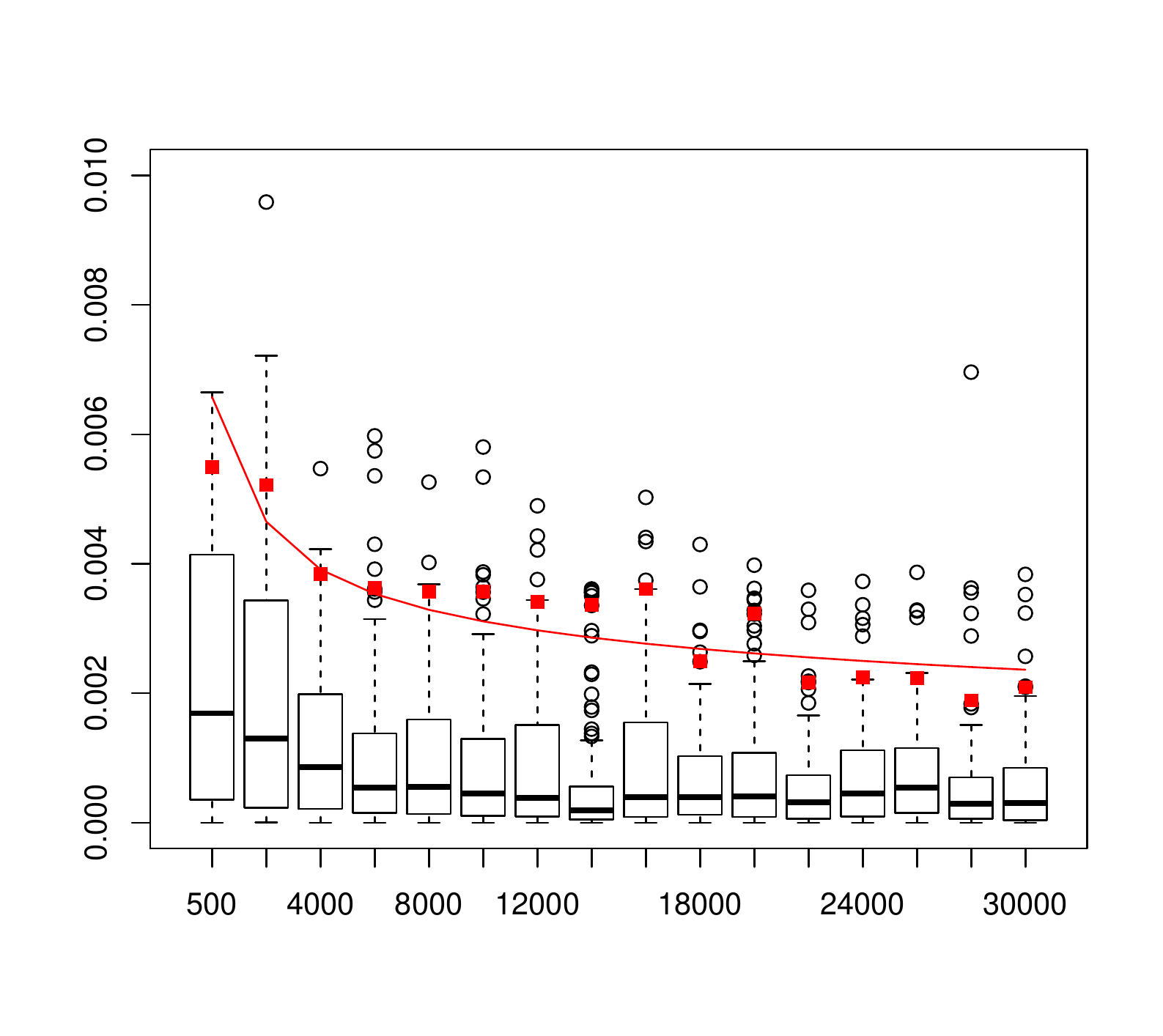}\includegraphics[width=4.2cm,height=7cm]{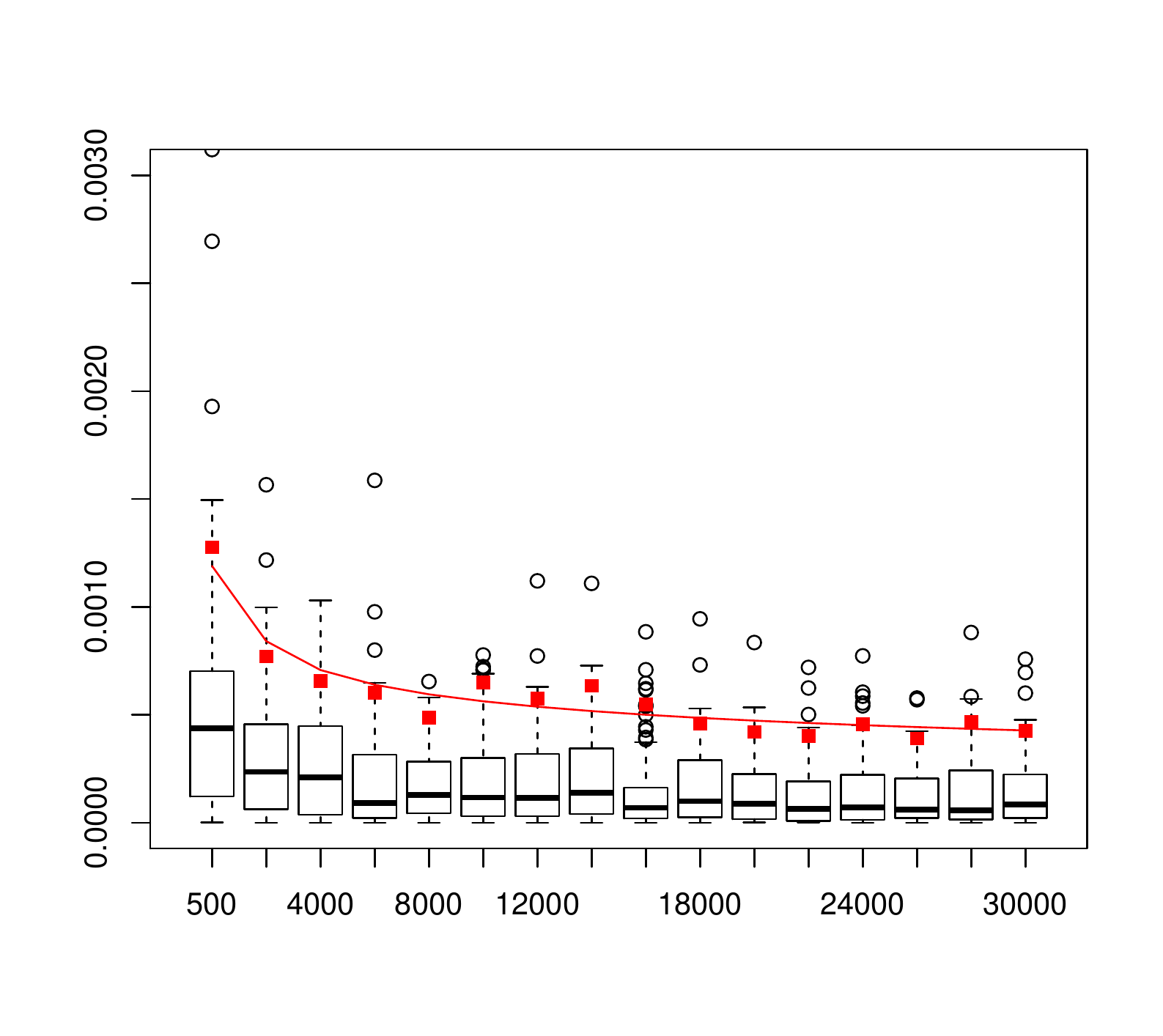}
\caption{Prediction errors as defined in Theorem \ref{leastsquareerror}. On the left, $K=2$ and $\pi=(0.8,0)$. On the right, $K=4$ and $\pi=(0.8, 0.6,0.4,0.2)$. 95\% quantiles are in red. Spearman rank correlation test detects in both cases that the quantiles decrease with $n$ (both p-values $< 2.10^{-16}$). The red line is obtained by a regression with respect to $n^{-1/4}$. }
\label{pred}
\end{figure}

Figure \ref{pred} shows that, as expected and for various set-up, the prediction error seems indeed to be decreasing in $n^{-1/4}$ (see Theorem \ref{leastsquareerror}).

\begin{figure}[h!]
\includegraphics[width=4.2cm,height=7cm]{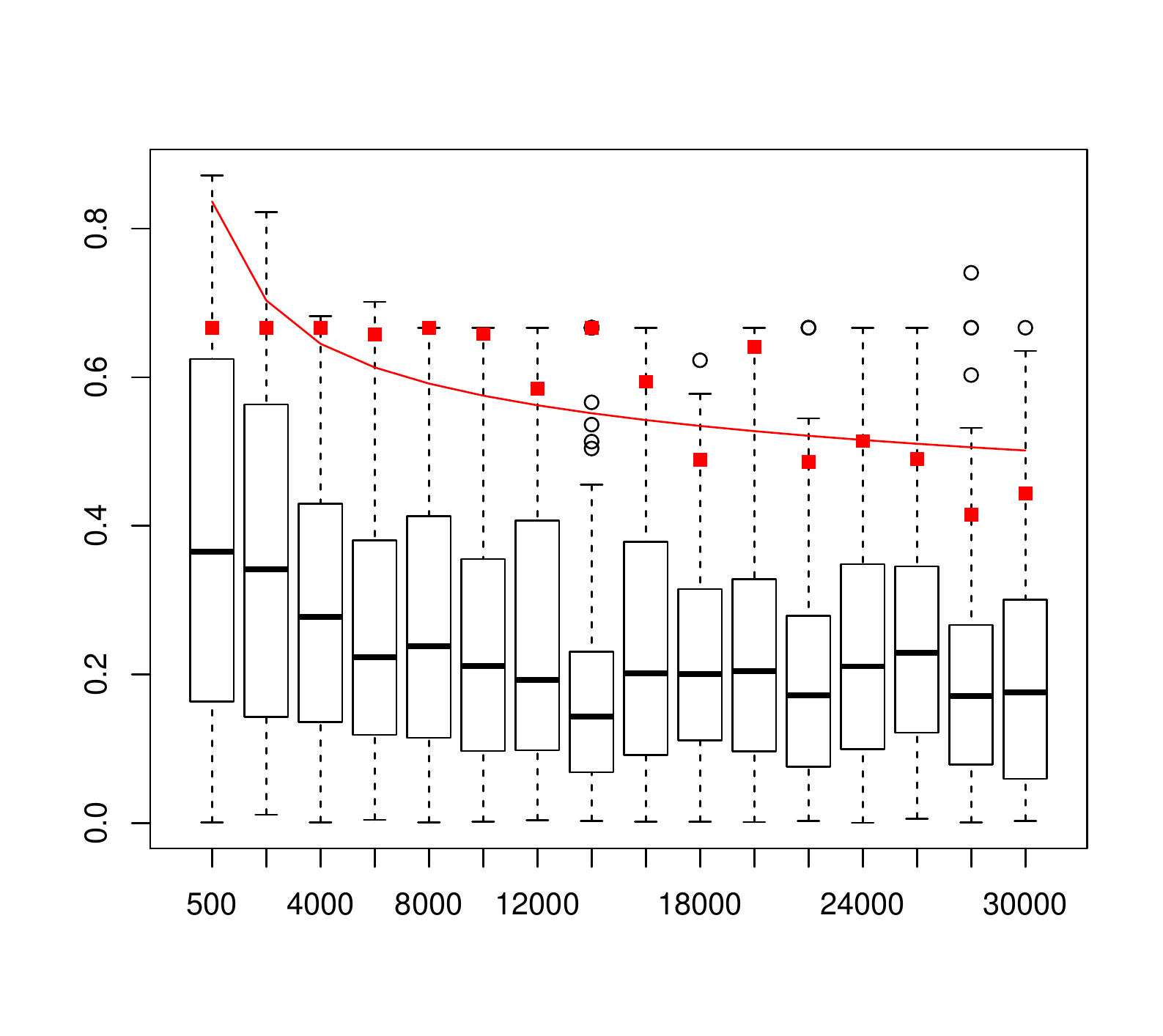}\includegraphics[width=4.2cm,height=7.4cm]{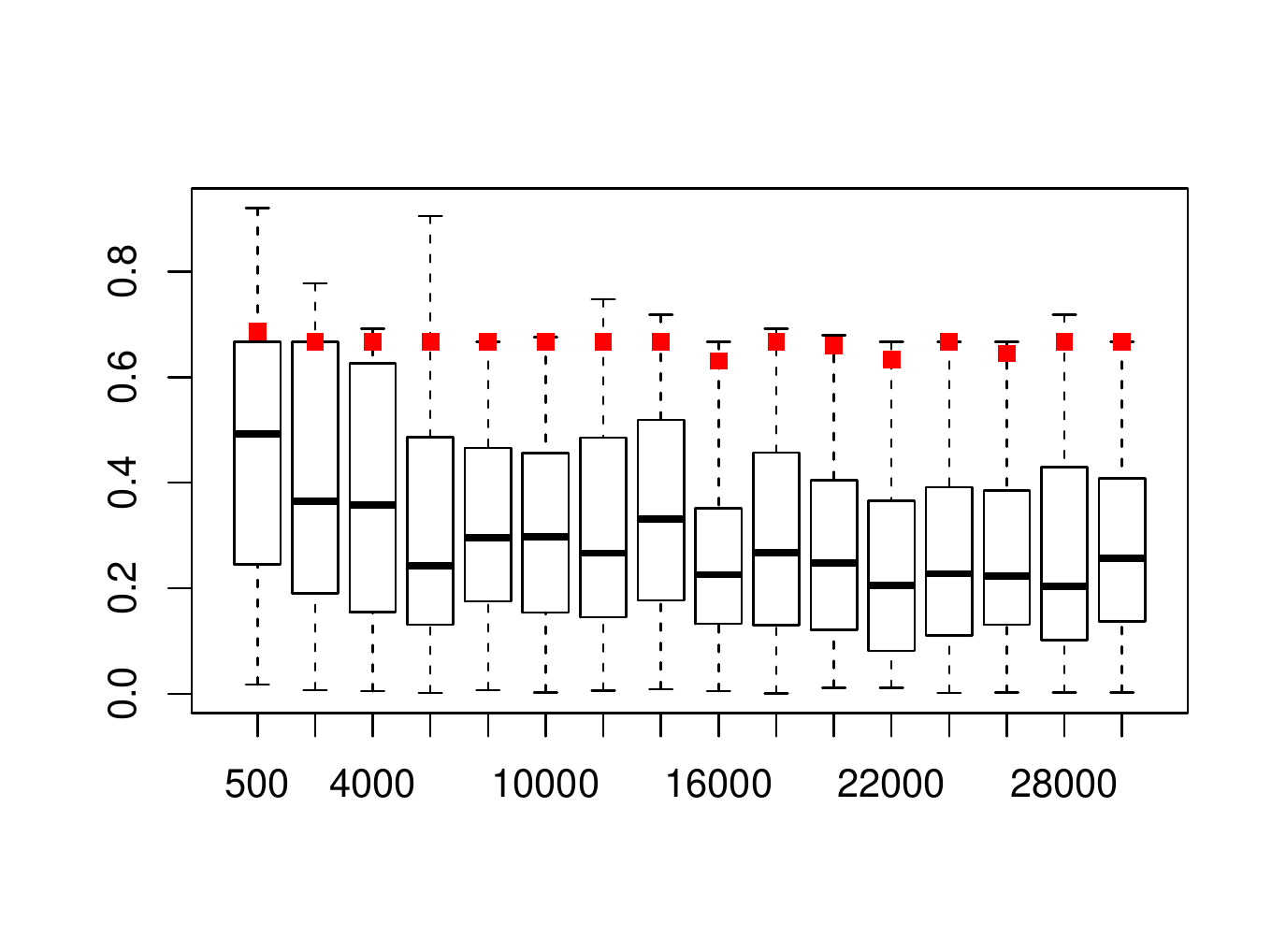}
\caption{Estimation errors $|\hat{\eta}-\eta|/\eta$. On the left, $K=2$ and $\pi=(0.8,0)$. On the right, $K=4$ and $\pi=(0.8, 0.6,0.4,0.2)$. 95\% quantiles are in red. Spearman rank correlation test detects in both cases that the quantiles decrease with $n$ (on the left p-value $<2.10^{-16}$, on the right p-value$=9. 10^{-5}$). On the left, the red line is obtained by a regression with respect to $n^{-1/8}$.  }
\label{est}
\end{figure}

The estimation error is illustrated in Figure \ref{est}. For the case $K=2$, $\pi=(0.8,0)$, we observe as expected that the error decreases in $n^{-1/8}$ (see Proposition \ref{probabilityestimationerror}). For the case $K=4$ and $\pi(0.8,0.6,0.4,0.2)$, which is not covered by  Proposition \ref{probabilityestimationerror},  we still observe a decreasing character of the 95\% quantile as a function of $n$ (p-value of the Spearman rank test in $9.10^{-5}$), but it seems to decrease much slower than $n^{-1/8}$.

\section{Conclusion}

In conclusion, we have shown that the estimation of the learning rate in Exp3 cannot be done correctly  if the true learning rate parameter is constant, that is the estimation rate is at most logarithmic. But the MLE on truncated observations can estimate adequately learning rates that are decreasing  at a polynomial rate with the number of observations. Note that the rate of convergence that we have shown either for the general prediction error or for the estimation error in particular cases are not the classic parametric rate.

Even if Exp3 is a toy learning model with respect to the vast literature in cognition, we believe that these phenomenons  appear in a large variety of models well used in practice such as for instance \cite{Kruschke1992} or \cite{GluckBower1988b} and that our theoretical conclusions should be kept in mind even when working on more realistic models from a more practical point of view.


\section*{Acknowledgements}

This work was supported by the French government, through the $\text{UCA}^\text{Jedi}$ and 3IA C\^ote d'Azur Investissements d'Avenir managed by the National Research Agency (ANR-15- IDEX-01 and ANR-19-P3IA-0002),  by the interdisciplinary Institute for Modeling in Neuroscience and Cognition (NeuroMod) of the Universit\'e C\^ote d'Azur and  directly by the National Research Agency (ANR-19-CE40-0024) with the ChaMaNe project. It is part of the \href{https://explain.i3s.univ-cotedazur.fr}{explAIn} project.

\bibliography{example_paper}
\bibliographystyle{icml2023}

\newpage
\appendix
\onecolumn

\appendix


\section{MISSING PROOFS}

  The Appendix contains the omitted proofs, as well as some indications on the numerical illustrations.

\subsection{For a constant learning rate}

In this section, we include the proofs and results used for the study of the Kullback-Leibler divergence for a constant learning rate. Technical lemmas used throughout the proofs are gathered in Section \ref{lemmasetaconstant}.

\subsubsection{Properties of \texorpdfstring{$q_i^\eta$}{the probabilities}}

\begin{lemma}
\label{summarypropertiesqi}
The following properties hold: 
\begin{itemize}
    \item[(i)] \label{decroissanceqieni} Let $\eta > 0$. The sequence $(q_i^\eta)_{i \geq 0}$ is decreasing to $0$ as $i \to + \infty$.
    \item[(ii)] \label{decroissanceqieneta} For all $i \in \N^*$, the function $\eta \longmapsto q_i^\eta$ is decreasing.
    \item[(iii)] \label{croissanceTi} The sequence $(T_i)_{i \geq 0}$ is non-decreasing. Moreover, $T_i \to + \infty$ as $i \to + \infty$ almost surely.
    \item[(iv)] \label{distributiongeometriqueTi} Let $\eta > 0$. Under the distribution with parameter $\eta$, for all $i \in \N$, the $T_{i+1}-T_i$ are independent and geometrically distributed with parameter $q_i^\eta$.
    \item[(v)] \label{limitesn} Let $S_n = \sum_{t=1}^n I_t$ be the number of times the worst arm is pulled before time $n$. Then, $S_n \to + \infty$ as $n \to + \infty$ almost surely.
    \item[(vi)] \label{decroissancep1} Let $\eta > 0$. The sequence $ (p_{1,t}^\eta)_{t \in \N^*}$ is decreasing to $0$ as $t \to + \infty$ almost surely.
\end{itemize}
\end{lemma}

\begin{proof}[Proof of (i)]
Let $\eta > 0$. We show that $(q_i^\eta)_{i \geq 0}$ is a decreasing sequence by induction.
Denote by $g$ the function defined for all $\gamma \geq 0$ and $q \geq 0$ by
\begin{equation}
\label{functiong}
\begin{cases}
    g(\gamma, 0) = 0, \\
    \displaystyle
    g(\gamma, q) = \frac{q e^{-\gamma \pi_1/q}}{(1-q)+q e^{-\gamma \pi_1/q}} 
        = \frac{1}{(\frac{1}{q}-1)e^{\gamma \pi_1/q}+1}
        \in (0,1] \quad \text{when $q>0$}.
\end{cases}
\end{equation}
This function is continuous in both its parameters.
The denominator is an increasing function of $1/q$, and thus $q \longmapsto g(\gamma,q)$ is an increasing function.

Moreover, $q_0^\eta = \frac{1}{2} > \frac{e^{-2 \eta \pi_1}}{1+e^{-2 \eta \pi_1}} = q_1^\eta$. Let $i \geq 0$, and assume $q_{i-1}^\eta\geq q_i^\eta$. Since $g$ is an increasing function of $q$, by the induction hypothesis,
\begin{equation*}
q_{i+1}^\eta = g(\eta, q_i^\eta) < g(\eta, q_{i-1}^\eta) = q_i^\eta.
\end{equation*}
Hence $(q_i^\eta)_{i \geq 0}$ is decreasing. Since it is lower bounded by 0, it converges to some $m \in [0,1/2]$. Since $g$ is continuous in its second parameter, $m = g(\gamma,m)$, that is
\begin{equation*}
m = \frac{m e^{-\pi_1 \eta/m}}{(1-m)+m e^{-\pi_1 \eta/m}},
\end{equation*}
or equivalently,
\begin{equation*}
m(1-m)(1-e^{-\pi_1 \eta/m}) = 0,
\end{equation*}
which admits only one solution in $[0,1/2]$: $m=0$.
\end{proof}

\begin{proof}[Proof of (ii)]
We show that $\eta \mapsto q_i^\eta$ is a decreasing function by induction on $i \in \N$. Let $\delta < \eta$. Firstly, for $i=0$, $q_0^\eta=q_0^\delta=1/2$. Let $i \in \N$, and assume that $q_i^\delta \geq q_i^\eta$. The function $g$ from \eqref{functiong} is decreasing in its first parameter and increasing in its second, therefore
\begin{equation*}
q_{i+1}^\eta = g(\eta, q_i^\eta)
    < g(\delta, q_i^\eta)
    \leq g(\delta, q_i^\delta)
    = q_{i+1}^\delta,
\end{equation*}
hence the result by induction.
\end{proof}

\begin{proof}[Proof of (iii)]
By definition of $T_i$, $T_{i+1} \geq T_i$ and $T_i \geq i$, hence the result.
\end{proof}

\begin{proof}[Proof of (iv)]
By definition, for all $t \in [T_i+1, T_{i+1}]$, $p_t^\eta = q_i^\eta$, and therefore, for all $s \geq 1$,
\begin{align*}
\PP^\eta(T_{i+1} - T_i = s \,|\, (T_{j+1} - T_j)_{j \neq i})
    &= \sum_{k \geq 1} \PP^\eta(I_{k+1}^{k+s-1} = 0, I_{k+s} = 1 \,|\, T_i=k, (T_{j+1} - T_j)_{j \neq i}) \PP^\eta(T_i=k \,|\, (T_{j+1} - T_j)_{j \neq i}) \\
    &= (1-q_i^\eta)^{s-1} q_i^\eta \sum_{k \geq 1} \PP^\eta(T_i=k \,|\, (T_{j+1} - T_j)_{j \neq i}) \\
    &= (1-q_i^\eta)^{s-1} q_i^\eta,
\end{align*}
which shows that $T_{i+1}-T_i$ is geometrically distributed with parameter $q_i^\eta$ and independent from the other $T_{j+1}-T_j$.
\end{proof}

\begin{proof}[Proof of (v)]
The sequence $(S_n)_{n \geq 1}$ is non-decreasing, so, almost surely, it either tends to $+\infty$ or it converges to some limit $m \in \N$.

Consider the event where $\lim_n S_n = m \in \N$, then on this event, $T_m < +\infty$ and $T_{m+1} = +\infty$, in particular $T_{m+1} - T_m = +\infty$, which is an event with probability zero by (iv) since $q_m^\eta > 0$ for all $m \in \N$. Therefore, $\Pbb^\eta(\lim_n S_n < +\infty) = 0$.
\end{proof}

\begin{proof}[Proof of (vi)]
First, note that $p_{1,t} = q^\eta_{S_{t-1}}$, where $(S_n)_{n \geq 0}$ is defined in (v), with the convention $S_0 = 0$. The sequence $(S_n)_{n \geq 1}$ is non-decreasing and tends to $+\infty$ almost surely by (v). Combining this with (i) shows that $(p_{1,t}^\eta)_{t \geq 1}$ is non-increasing and converges to 0 almost surely.
\end{proof}

\subsubsection{Proof of Proposition~\ref{afterIeta}: tetration behaviour of the updated probabilities.}
\label{proofoftetration}

Let us show by induction that for all $k \geq 1$ : 
\begin{equation*}
    q_{I(\eta)+k+1}^\eta \leq \frac{2 \min(\eta \pi_1, q_{I(\eta)+k}^\eta)  }{f^{(k)  }(2)} \leq \frac{1  }{f^{(k)  }(2)} .
\end{equation*}
For any $\eta > 0$ and $i \geq 0$, $1-q_{i}^\eta + q_i^\eta e^{-\eta \pi_1/q_i^\eta} \geq \frac{1}{2} $ because $q_i^\eta \leq \frac{1}{2}$. Therefore, $q_{i+1}^\eta \leq 2 q_i^\eta e^{-\eta \pi_1 / q_i^\eta}$. In particular,
\begin{equation*}
\begin{split}
    q_{I(\eta)+2}^{\eta} & \leq \frac{q_{I(\eta)+1}^{\eta} e^{-\eta \pi_1 / q_{I(\eta)+1}^\eta}}{1-q_{I(\eta)+1}^{\eta} + q_{I(\eta)+1}^{\eta} e^{-\eta \pi_1 / q_{I(\eta)+1}^\eta} } \\& \leq 2 q_{I(\eta)+1}^{\eta} e^{-\eta \pi_1 / q_{I(\eta)+1}^\eta}  \\& \leq 2 \min(\eta \pi_1, q_{I(\eta)+1}^\eta) \cdot e^{-1}  = \frac{2 \min(\eta \pi_1, q_{I(\eta)+1}^\eta) }{f\big(2\big)} \leq \frac{1 }{f\big(2\big)}.
\end{split}
\end{equation*}
Let $k \geq 2$. Assume the result is true for $k \geq 2$. Then,
\begin{equation*}
\begin{split}
    q_{I(\eta)+k+2}^{\eta} & \leq 2  q_{I(\eta)+k+1}^{\eta} e^{-\eta \pi_1 / q_{I(\eta)+k+1}^{\eta}} \\& \leq 2 \min(\eta \pi_1, q_{I(\eta)+k+1}^{\eta}) \cdot e^{-\eta \pi_1 f^{(k)  }(2)/ (2 \eta \pi_1)  } 
    \\& \leq 2 \min(\eta \pi_1, q_{I(\eta)+k+1}^{\eta}) \cdot \frac{1}{e^{f^{(k)  }(2)/2}}  \\& \leq  \frac{2 \min(\eta \pi_1, q_{I(\eta)+k+1}^{\eta})}{ f^{(k+1)  }\big(2\big)} \leq   \frac{1}{ f^{(k+1)  }\big(2\big)} .
\end{split}
\end{equation*}
Let us study the monotonicity of the sequence $(f^{(k)}(2))_{k \geq 0}$: $f^{(0)}(2) = 2 < e = f^{(1)}(2)$. Suppose the result is true for $k \in \N$. Then, since $f$ is increasing, 
\begin{equation*}
f^{(k)}(2) < f^{(k+1)}(2)
    \Longrightarrow f^{(k+1)}(2) < f^{(k+2)}(2).
\end{equation*}
Hence, $(f^{(k)}(2))_{k \geq 0}$ is increasing.

For the second part of the Proposition, $\frac{1}{f^{(k)}(2)} < \frac{1}{n}$ implies $q_{I(\eta)+k+1}^\eta < \frac{1}{n}$, that is $k+I(\eta)+1 \geq J(n,\eta)+1$. Let $\log^*(n)$ be defined as in Proposition~\ref{afterIeta}. For all $k \geq 1$,
\begin{equation*}
\begin{split}
   \frac{1}{f^{(k)}(2)} < \frac{1}{n} \Longleftrightarrow n < f^{(k)}(2)
   \Longleftrightarrow k \geq \log^*(n) + 1,
\end{split}
\end{equation*}
because the sequence $\big(f^{(k)}(2)\big)_{k \geq 1}$ is increasing. In particular, taking $k = \log^*(n)+1$,
\begin{equation*}
    J(n,\eta) \leq I(\eta) + \log^*(n)+1.
\end{equation*}

\subsubsection{Proof of Theorem~\ref{boundingklconstant}}
\label{proofoftheoremetaconstant}

The proof is divided into two parts. We first show that the KL divergence can be bounded, up to an additive constant, by a sum whose number of terms is bounded by $J(n,\delta)$, that is:
\begin{equation*}
    KL(\PP_{I_1^n}^\eta || \PP_{I_1^n}^\delta)  \leq \sum_{i=0}^{J(n,\delta)} 2 \frac{ q_i^\delta-q_i^\eta}{q_i^\eta} + n (J(n,\delta)+1) (1-q_{J(n,\delta)}^\eta)^{\displaystyle \lfloor \frac{n}{J(n,\delta)+1} \rfloor} +  O_n(1).
\end{equation*}
We then show the existence of $\delta >0$ and $\eta > \delta$, as in the theorem, that verifies that the KL divergence is bounded.

\begin{proof}[Proof of Theorem \ref{boundingklconstant}.]
\label{proofwhenunderestimating}
Let us split the KL divergence into three parts. Assume in the following that $\eta > \delta$, so that $\log \frac{q_i^\eta}{q_i^\delta}\leq 0$. Recall from Equation~\eqref{expressionoftheKLdivergence} that
\begin{align*}
KL(\PP_{I_1^n}^\eta || \PP_{I_1^n}^\delta)
    &= \sum_{i=0}^{n-1}\EE_{\eta} [\one_{T_{i+1} \leq n} (T_{i+1}-T_i-1) \log \frac{1-q_i^\eta}{1-q_i^\delta}+\one_{T_{i+1} \leq n} \log \frac{q_i^\eta}{q_i^\delta} \\
    &\quad + \EE_{\eta} \bigg[ (n-T_{\tau(n)})\log \frac{1-q_{\tau(n)}^\eta}{1-q_{\tau(n)}^\delta} \bigg] \\
    &= \underbrace{\sum_{i=0}^{J(n,\delta)}\EE_{\eta} [\one_{T_{i+1} \leq n} (T_{i+1}-T_i-1) \log \frac{1-q_i^\eta}{1-q_i^\delta}+\one_{T_{i+1} \leq n} \log \frac{q_i^\eta}{q_i^\delta}]}_{A} \\
    &\quad + \underbrace{\sum_{i=J(n,\delta)+1}^{n-1}\EE_{\eta} [\one_{T_{i+1} \leq n} (T_{i+1}-T_i-1)  \log \frac{1-q_i^\eta}{1-q_i^\delta}+\one_{T_{i+1} \leq n} \log \frac{q_i^\eta}{q_i^\delta}]}_{B} \\
    &\quad + \underbrace{\EE_{\eta} \bigg[ (n-T_{\tau(n)})\log \frac{1-q_{\tau(n)}^\eta}{1-q_{\tau(n)}^\delta} \bigg]}_{C}.
\end{align*}

Since $0 \leq q_i^\eta \leq q_i^\delta$ and $\log(1-\cdot)$ is non-increasing and $2$-lipschitz on $[0,\frac{1}{2}]$, for all $i \geq 0$, if $\delta \pi_1 n \geq \log 2$ and $n \geq 3$,
\begin{align*}
B
    &\leq \sum_{i=J(n,\delta)+1}^{n-1} \EE_{\eta} \left[\one_{T_{i+1} \leq n} (T_{i+1} - T_i) \right] \log \frac{1-q_i^\eta}{1-q_i^\delta} \\
    &\leq \sum_{i=J(n,\delta)+1}^{n-1} n \PP_{\eta}(T_{i+1} \leq n) \times 2(q_i^\delta-q_i^\eta) \\
    &\leq 2n \sum_{i=J(n,\delta)+1}^{n-1} q_i^\delta \\
    &\leq 2n \frac{1}{n-2} \sum_{i \geq 0} e^{-\delta \pi_1 n i} \quad \text{by Lemma~\ref{indicetradeoff}}\\
    &\leq \frac{6}{1 - e^{-\delta \pi_1 n}} \leq 12.
\end{align*}

For the same reasons,
\begin{align*}
C
    &= \EE_\eta [(n-T_{\tau(n)}) \log \frac{1-q_{\tau(n)}^\eta}{1-q_{\tau(n)}^\delta}] \\
    &\leq 2n \EE_\eta [q_{\tau(n)}^\delta] \\
    &= 2n \EE_\eta [q_{\tau(n)}^\delta \one_{\tau(n) \leq J(n,\delta)}]
        + 2n \EE_\eta [q_{\tau(n)}^\delta \one_{\tau(n) \geq J(n,\delta)+1}] \\
    &\leq 2n \times \frac{1}{2} \PP_\eta(\tau(n) \leq J(n,\delta))
        + 2 \quad \text{by definition of $J(n,\delta)$} \\
    &= n \PP_\eta(T_{J(n,\delta)+1} > n) + 2 \\
    &\leq n (J(n,\delta)+1) (1-q_{J(n,\delta)}^\eta)^{\displaystyle \lfloor \frac{n}{J(n,\delta)+1} \rfloor} + 2 \quad \text{by~\eqref{inequalityTi}}
\end{align*}

Finally, provided that $\eta \pi_1 n \geq \log 2$ and $n \geq 3$,
\begin{align*}
A
    &\leq \sum_{i=0}^{J(n,\delta)} \EE_{\eta} [\one_{T_{i+1} \leq n} (T_{i+1}-T_i-1)] \log \frac{1-q_i^\eta}{1-q_i^\delta} \\
    &\leq 2 \sum_{i=0}^{J(n,\delta)} (q_i^\delta - q_i^\eta) \left\{ \EE_{\eta} [T_{i+1}-T_i] \wedge n \PP_\eta(T_{i+1} \leq n) \right\} \\
    &\leq 2 \sum_{i=0}^{J(n,\delta)} (q_i^\delta - q_i^\eta) \left\{ \frac{1}{q_i^\eta} \wedge n^2 q_i^\eta \right\}  \quad \text{by~\eqref{inequalityTi}} \\
    &\leq 2 \sum_{i=0}^{(J(n,\eta)+1) \wedge J(n,\delta)} \frac{q_i^\delta - q_i^\eta}{q_i^\eta}
        + 2 \frac{n^2}{n-2} \sum_{i=J(n,\eta)+2}^{+\infty} e^{-\eta \pi_1 n (i-J(n,\eta)-1)} \quad \text{by Lemma~\ref{indicetradeoff}} \\
    &\leq 2 \sum_{i=1}^{(J(n,\eta)+1) \wedge J(n,\delta)} \frac{q_i^\delta - q_i^\eta}{q_i^\eta}
        + 12 n e^{-\eta \pi_1 n}.
\end{align*}

Therefore,
\begin{equation*}
KL(\PP_{I_1^n}^\eta || \PP_{I_1^n}^\delta)
    \leq 2 \sum_{i=1}^{(J(n,\eta)+1) \wedge J(n,\delta)} \frac{ q_i^\delta-q_i^\eta}{q_i^\eta}
        + n (J(n,\delta)+1) (1-q_{J(n,\delta)}^\eta)^{\displaystyle \lfloor \frac{n}{J(n,\delta)+1} \rfloor}
        + 14 + 12 n e^{-\eta \pi_1 n}.
\end{equation*}

From Lemma \ref{decroissanceratioqi}, by monotonicity of the sequence $(q^\delta_k / q^\eta_k)_{k \geq 0}$, 
\begin{equation}
\label{KLdernierJ}
KL(\PP_{I_1^n}^\eta || \PP_{I_1^n}^\delta) \leq 2 J(n,\delta) \frac{q_{J(n,\delta)}^\delta-q_{J(n,\delta)}^\eta}{q_{J(n,\delta)}^\eta}
    + n (J(n,\delta)+1) (1-q_{J(n,\delta)}^\eta)^{\displaystyle \lfloor \frac{n}{J(n,\delta)+1} \rfloor}
    + 14 + 12 n e^{-\eta \pi_1 n}.
\end{equation}

By Proposition~\ref{afterIeta}, for any $\delta > 0$ such that $\delta \pi_1 \geq \frac{1}{n}$, 
\begin{equation*}
    I(\delta) \leq J(n,\delta) \leq I(\delta) + \log^*(n) + 1.
\end{equation*}

Let $R > 0$ be such that $\frac{1}{n} < R \pi_1 \leq 1/2$. By the inequalities above, for any $\delta$ such that $\frac{1}{n}< \delta \pi_1 < R \pi_1$,
\begin{equation*}
J(n,\delta) - J(n,R)
    \geq I(\delta) - \Big(I(R) + \log^*(n) + 1\Big).
\end{equation*}
From Lemma \ref{encadrementI(eta)},
$I(\delta) > I(R) + \log^*(n) + 1$ (and thus $J(n,\delta) > J(n,R)$) as soon as
\begin{equation*}
    \frac{1}{2 \delta \pi_1} - 2 > \frac{2}{R \pi_1} - 4 + \log^*(n) + 1,
\end{equation*}
that is when
\begin{equation*}
    \delta < \frac{\frac{1}{2\pi_1}}{\frac{2}{R\pi_1} + \log^*(n) - 1}.
\end{equation*}
Therefore, by contraposition, choosing $\delta$ such that $J(n,\delta) = J(n,R)$ implies $\delta \geq \displaystyle \frac{\frac{1}{2\pi_1}}{\frac{2}{R\pi_1} + \log^*(n) - 1}$. By continuity of $\gamma \mapsto q_{J(n,R)+1}^{\gamma}$ and the fact that its limit is $1/2$ (resp. smaller than $1/n$) when $\gamma \rightarrow 0$ (resp. $R$), for all $n \geq 2$, there exists $\delta \in (0,R)$ such that $q_{J(n,R)+1}^{\delta} \in [\frac{1}{2n}, \frac{1}{n})$. 
Let $\delta_n$ be such a $\delta$. In particular, $J(n,\delta_n) \leq J(n,R)$, thus is equal since $\delta \mapsto J(n,\delta)$ is non-increasing. Therefore,
\begin{equation}
\label{choixconstantec}
\delta_n
    \geq \frac{\frac{1}{2\pi_1}}{\frac{2}{R\pi_1} + \log^*(n) - 1}
    \geq \frac{(4\pi_1)^{-1}}{\log^*n}
    =: \frac{c}{\log^*n}.
\end{equation}
for $n$ such that $\log^*(n) \geq \frac{2}{R \pi_1} - 1$. As a consequence, $\delta_n \pi_1 n \geq 1 \vee \log 2$ for all $n \geq n_0$, for some $n_0$ that depends only on $R$ and $\pi_1$.
From Lemma~\ref{lipschitzianityJ(n,delta)}, for any $\eta > \delta_n$,
\begin{equation*}
q_{J(n,\delta_n)}^{\delta_n} - q_{J(n,\delta_n)}^\eta
    \leq G(n,\delta_n) (\eta-\delta_n), 
\end{equation*}
where
\begin{equation*}
G(n,\delta_n)
    = \frac{8}{\delta_n \pi_1} \frac{\left(\frac{8}{\delta_n \pi_1}\right)^{J(n,\delta_n)}-1}{\frac{8}{\delta_n \pi_1}-1}
    \leq 2 \left(\frac{8}{\delta_n \pi_1}\right)^{J(n,\delta_n)}
\end{equation*}
In particular, for $\eta_n = \delta_n + \frac{1}{(\log n)^{1+\beta}}$,
\begin{equation}
\label{eq1}
q_{J(n,\delta_n)}^{\delta_n} - q_{J(n,\delta_n)}^{\eta_n}
    \leq \frac{G(n,\delta_n)}{(\log n)^{1+\beta}}.
\end{equation}
Note that $\frac{G(n,\delta_n)}{(\log n)^{\alpha}} \to 0$ as $n \to + \infty$ for any $\alpha > 0$. Indeed, by Proposition~\ref{afterIeta} and Lemma~\ref{encadrementI(eta)},
\begin{equation}
\label{eq_majoration_Jdelta}
J(n,\delta_n) + 1
    \leq I(\delta_n)+\log^*(n) + 2
    \leq \frac{2}{\delta_n \pi_1} - 4 + \log^*(n) + 2
    \leq 9 \log^*(n)
\end{equation}
by~\eqref{choixconstantec}, as long as $n$ is such that $\log^*(n) \geq \frac{2}{R\pi_1} - 1$, which also entails $\frac{1}{\delta_n \pi_1} \leq 4\log^*(n)$, so that $\log G(n,\delta_n) = O(\log^*(n) \log \log^*(n)) = o(\log \log n)$.

Therefore, since $q_{J(n,\delta_n)}^{\delta_n} \geq \frac{\delta \pi_1}{\log(2n)}$ by Lemma~\ref{tradeoffprlesipetits} and the assumption $q_{J(n,\delta_n)+1}^{\delta_n} \geq 1/(2n)$,
\begin{equation}
\label{eq2}
q_{J(n,\delta_n)}^{\eta_n}
    \geq \frac{\delta_n \pi_1}{\log(2n)} - \frac{ G(n,\delta_n) }{(\log n)^{1+\beta}}
    \geq \frac{1}{8 \log^*(n) \log n}
\end{equation}
for $n \geq n_0$, for some $n_0$ depending on $R$, $\pi_1$ and $\beta$.
Therefore, by~\eqref{eq_majoration_Jdelta},
\begin{align*}
n (J(n,\delta_n)+1) (1- q_{J(n,\delta_n)}^{\eta_n})^{\displaystyle \lfloor \frac{n}{J(n,\delta_n)+1} \rfloor}
    &\leq n (J(n,\delta_n)+1) \left(1-\frac{1}{8 \log^*(n) \log n}\right)^{\displaystyle \lfloor \frac{n}{J(n,\delta_n)+1} \rfloor} \\
    &\leq n (J(n,\delta_n)+1) \exp \left(- \lfloor \frac{n}{J(n,\delta_n)+1} \rfloor \frac{1}{8 \log^*(n) \log n} \right) \\
    &\leq n \times 9 \log^*(n) \exp \left(- \lfloor \frac{n}{9 \log^*(n)} \rfloor \frac{1}{8 \log^*(n) \log n} \right) \\
    &= o_n(1),
\end{align*}
where we used $\log(1-x) \leq -x$ for $x \in (0,1)$.
Injecting this result in~\eqref{KLdernierJ}, and since $\eta_n \pi_1 n \geq \log n$ for $n \geq n_0'$ for some $n_0'$ that depends only on $\pi_1$ and $\beta$,
\begin{align*}
KL(\PP_{I_1^n}^{\eta_n} || \PP_{I_1^n}^{\delta_n})
    &\leq 2 J(n,\delta) \frac{q_{J(n,\delta_n)}^{\delta_n} - q_{J(n,\delta_n)}^{\eta_n}}{q_{J(n,\delta_n)}^{\eta_n}}
    + 14 + 12 n e^{-\pi_1 n / (\log n)^{1+\beta}} + o_n(1).
\end{align*}
Using~\eqref{eq1} and~\eqref{eq2}, for $n$ large enough,
\begin{align*}
KL(\PP_{I_1^n}^{\eta_n} || \PP_{I_1^n}^{\delta_n})
    &\leq 2 J(n,\delta_n) \frac{\frac{G(n,\delta_n)}{(\log n)^{1+\beta}}}{\frac{1}{8 \log^*(n) \log n}} + 14 + o_n(1). \\
    &\leq 18 \log^*(n) \frac{8 \log^*(n) G(n,\delta_n)}{(\log n)^{\beta}} + 14 + o_n(1) \\
    &= 14 + o_n(1)
\end{align*}
since $G(n,\delta_n) \vee \log^*(n) = o_n((\log n)^{\alpha})$ for any $\alpha > 0$. Note that this $o_n(1)$ only depends on $\pi_1$ and $\beta$. Therefore, there exists $\Delta_{\beta,\pi_1}>0$ depending on $\pi_1$ and $\beta$ and $n_0 \in \N$ depending on $\pi_1$, $R$ and $\beta$, such that for all $n \geq n_0$,
\begin{equation*}
    KL(\PP_{I_1^n}^{\eta_n} || \PP_{I_1^n}^{\delta_n}) \leq \Delta_{\beta,\pi_1}.
\end{equation*}

\end{proof}

\subsubsection{Technical lemmas}
\label{lemmasetaconstant}

\begin{lemma}[Bracketing the c.d.f. of $T_i$]
Let $\eta > 0$.
For all $i \in \N$,
\begin{equation}
\label{inequalityTi}
1-i(1-q^\eta_{i-1})^{\lfloor \frac{n}{i} \rfloor}
    \leq \PP^\eta(T_i \leq n)
    \leq 1 - (1-q^\eta_{i-1})^n
    \leq n q^\eta_{i-1}.
\end{equation}
\end{lemma}

\begin{proof}
Since $T_i = T_i - T_0 = \sum_{k=0}^{i-1} (T_{k+1}-T_k) \geq T_{i}-T_{i-1}$,
\begin{align*}
\PP^\eta(T_i \leq n)
    &= \PP^\eta(\sum_{k=0}^{i-1} T_{k+1}-T_k \leq n) \\
    &\leq \PP^\eta(T_i-T_{i-1} \leq n) \\
    &= 1-(1-q^\eta_{i-1})^n \quad \text{by Proposition~\ref{summarypropertiesqi} (iv)} \\
    &\leq 1 - (1-n q^\eta_{i-1}) = n q^\eta_{i-1} \quad \text{by Bernoulli's inequality.}
\end{align*}
On the other hand,
\begin{align*}
\PP^\eta(T_i \leq n)
    \geq \PP^\eta(\forall k \in \{0, \dots, i-1\}, T_k - T_{k-1}\leq \lfloor  \frac{n}{i} \rfloor)
    &\geq 1 - \sum_{k=0}^{i-1} (1-q^\eta_{k-1})^{\lfloor \frac{n}{i} \rfloor} \quad \text{by union bound} \\
    &\geq 1 - i(1-q^\eta_{i-1})^{\lfloor \frac{n}{i} \rfloor},
\end{align*}
which concludes the proof.
\end{proof}


\begin{lemma}
\label{indicetradeoff}
For all $n \geq 3$ such that $\eta \pi_1 n \geq \log 2$, for all $i \geq J(n,\eta)+1$,
\begin{equation*}
q_{i}^\eta \leq \frac{1}{n-2} e^{-\eta \pi_1 n (i - J(n,\eta) - 1)}.
\end{equation*}
In particular, under the same conditions on $n$ and $\eta$, for all $i \geq J(n,\eta)+2$,
\begin{equation*}
q_{i}^\eta \leq \frac{e^{-\eta \pi_1 n}}{n-2} \leq e^{-\eta \pi_1 n}.
\end{equation*}
\end{lemma}

\begin{proof}
Write $h : (x,y) \in (0,1) \times \R_+ \longmapsto \frac{x e^{-y}}{1 - x}$, so that $q_{i+1}^\eta \leq h(q_i^\eta, \frac{\eta \pi_1}{q_i^\eta})$. This function $h$ is increasing in $x$ and decreasing in $y$.

Now, let $k \geq 1$, and assume that there exists $0 \leq c_k < n$ such that $q_{J(n,\eta) + k}^\eta \leq \frac{1}{n - c_k} e^{-\eta \pi_1 n (k-1)}$. By definition of $J(n,\eta)$, this holds for $k=1$ with $c_1 = 0$.

Since $i \mapsto q_i^\eta$ is decreasing and $J(n,\eta)+k+1 > J(n,\eta)$, $q_{J(n,\eta) + k+1}^\eta \leq \frac{1}{n}$, so
\begin{align*}
q_{J(n,\eta) + k+1}
    \leq h(q_{J(n,\eta) + k}, \frac{\eta \pi_1}{1/n})
    &= \frac{1}{n-c_k} e^{-\eta \pi_1 n k} \frac{1}{1 - \frac{1}{n-c_k} e^{-\eta \pi_1 n (k-1)}} \\
    &= \frac{e^{-\eta \pi_1 n k}}{n - c_k- e^{-\eta \pi_1 n (k-1)}}
\end{align*}
thus the same inequality holds for $q_{J(n,\eta) + k+1}$ with $c_{k+1} = c_k + e^{-\eta \pi_1 n (k-1)}$, and for all $k \geq 1$, $c_k \leq \frac{1}{1 - e^{-\eta \pi_1 n}}$. Since $\eta \pi_1 n \geq \log 2$, $c_k \leq 2 < n$ for all $k \geq 1$ (and $n \geq 3$), which concludes the proof.
\end{proof}

\begin{lemma}
\label{tradeoffprlesipetits}
Let $\eta > 0$.
For all $i \geq 0$,
\begin{equation*}
    q_i^\eta \geq \frac{\eta \pi_1}{\log(1/q^\eta_{i+1})}.
\end{equation*}
In particular, for all $i \leq J(n,\eta)-1$,
\begin{equation*}
    q_i^\eta \geq \frac{\eta \pi_1}{\log n}.
\end{equation*}
\end{lemma}

\begin{proof}
First, for any $i \geq 1$,
\begin{equation*}
q_i^\eta
    = \frac{q_{i-1}^\eta e^{-\eta \pi_1/ q_{i-1}^\eta}}{1-q_{i-1}^\eta + q_{i-1}^\eta e^{-\eta \pi_1/ q_{i-1}^\eta}}
    \leq \frac{q_{i-1}^\eta}{1-q_{i-1}^\eta} e^{-\eta \pi_1/ q_{i-1}^\eta}
    \leq e^{-\eta \pi_1/ q_{i-1}^\eta}
\end{equation*}
since $q_{i-1}^\eta \leq 1/2$, which leads to the first inequality.

For the second, recall that by definition, for any $i \leq J(n,\eta)$, $q_i^\eta \geq \frac{1}{n}$, and use the first inequality.
\end{proof}



\begin{lemma}\label{decroissanceratioqi}
Let $\delta > \eta > 0$. Then, the sequence $ \displaystyle\Big(\frac{q_i^\delta}{q_i^\eta}\Big)_{i \geq 0}$ is decreasing and tends to $0$ when $i$ tends to $+ \infty$.
\end{lemma}

\begin{proof}
Let $u_i := \displaystyle \frac{q_i^\delta}{q_i^\eta}$ for all $i \geq 0$. Then, by definition of $q_i^\eta$,
\begin{equation*}
    u_{i+1} = u_i \cdot  \frac{e^{-\delta \pi_1 / q_i^\delta}}{(1-q_i^\delta) + q_i^\delta e^{-\delta \pi_1/ q_i^\delta}} \cdot \frac{ (1-q_i^\eta) + q_i^\eta e^{-\eta \pi_1/ q_i^\eta} }{e^{-\eta \pi_1 / q_i^\eta}}.
\end{equation*}
Let $h(q,\gamma) \displaystyle := \frac{e^{-\gamma \pi_1 / q}}{(1-q) + q e^{-\gamma \pi_1/ q}}$.
Then, for all $0< q < 1$ and $\gamma > 0$,
\begin{align*}
    \frac{\partial h}{\partial \gamma} (q,\gamma) & = \frac{-\frac{\pi_1}{q}e^{-\gamma \pi_1/q}\Big(1-q+q e^{-\gamma \pi_1/q} \Big) - e^{-\gamma \pi_1/q} (-\pi_1 e^{-\gamma \pi_1/q} ) }{\Big(1-q+q e^{-\gamma \pi_1/q} \Big)^2}
    \\& = \frac{-\frac{\pi_1}{q}e^{-\gamma \pi_1/q}\Big(1-q\Big) }{\Big(1-q+q e^{-\gamma \pi_1/q} \Big)^2} < 0.
\end{align*}
And, for all $0< q < 1$ and $\gamma > 0$,
\begin{align*}
    \frac{\partial h}{\partial q} (q,\gamma) & = \frac{\frac{\gamma \pi_1}{q^2} e^{-\gamma \pi_1/q} \Big(1-q+q e^{-\gamma \pi_1/q} \Big) - e^{-\gamma \pi_1/q} \Big(-1 + e^{-\gamma \pi_1/q} (1+ \frac{ \gamma \pi_1}{q})\Big)}{\Big(1-q+q e^{-\gamma \pi_1/q} \Big)^2}
    \\& = \frac{\frac{\gamma \pi_1}{q^2} e^{-\gamma \pi_1/q} (1-q) -e^{-\gamma \pi_1/q}(-1 + e^{-\gamma \pi_1/q}) }{\Big(1-q+q e^{-\gamma \pi_1/q} \Big)^2}
    \\& = e^{-\gamma \pi_1/q} \frac{\frac{\gamma \pi_1}{q^2}  (1-q) + 1 - e^{-\gamma \pi_1/q} }{\Big(1-q+q e^{-\gamma \pi_1/q} \Big)^2} > 0.
\end{align*}
Hence $h$ is decreasing in $\gamma$ and increasing in $q$.
Let $i \geq 0$. 
Since $\delta > \eta$, $q_i^\delta \leq q_i^\eta$ and $u_i > 0$,
\begin{equation*}
0 < \frac{u_{i+1}}{u_i}
    = \frac{h(q_i^\delta,\delta)}{h(q_i^\eta,\eta)}
    \leq \frac{h(q_i^\eta,\delta)}{h(q_i^\eta,\eta)}
\begin{cases}
    < \frac{h(q_i^\eta,\eta)}{h(q_i^\eta,\eta)} = 1 \\
    = \frac{e^{-\delta \pi_1/ q_i^\eta}}{1-q_i^\eta + q_i^\eta e^{-\delta \pi_1/q_i^\eta}} \cdot \frac{1-q_i^\eta + q_i^\eta e^{-\eta \pi_1/q_i^\eta}}{{e^{-\eta \pi_1/ q_i^\eta }}}
    \sim e^{-\frac{(\delta-\eta)}{q_i^\eta} \pi_1} 
    \longrightarrow 0
\end{cases}
\end{equation*}
since $q_i^\eta \longrightarrow 0$ when $i \to + \infty$. Therefore, $(u_i)_{i \geq 1}$ is decreasing and tends to $0$.
\end{proof}

\begin{lemma}
\label{encadrementI(eta)}
Let $\eta >0$ such that $\eta \pi_1 < 1$. For all $i \leq I(\eta)$,
\begin{equation}
\label{qilineaire}
\frac{1}{2}-i \eta \pi_1
    \leq  q_i^{\eta}
    \leq 
        \frac{1}{2} - i \frac{\eta \pi_1}{4}.
\end{equation}
Therefore,
\begin{equation*}
\frac{1}{2 \eta \pi_1} - 2
    \leq I(\eta)
\end{equation*}
and when $\eta \pi_1 \leq \frac{1}{2}$,
\begin{equation*}
I(\eta)
    \leq \frac{2}{\eta \pi_1} - 4.
\end{equation*}
\end{lemma}

\begin{proof}
Let's prove~\eqref{qilineaire} first, starting with the lower bound. For all $i$,
\begin{align*}
q_i^{\eta} - q_{i+1}^{\eta}
    = \frac{q_i^{\eta}(1-q_i^\eta) (1 - e^{-\eta \pi_1 / q_i^{\eta}})}{1-q_i^{\eta}+ q_i^{\eta} e^{-\eta \pi_1 / q_i^{\eta}}}
    \leq q_i^{\eta} (1 - e^{-\eta \pi_1 / q_i^{\eta}})
    \leq q_i^{\eta} \frac{\eta \pi_1}{q_i^{\eta}}
    = \eta \pi_1
\end{align*}
since $\frac{1-q_i^\eta}{1-q_i^{\eta}+ q_i^{\eta} e^{-\eta \pi_1 / q_i^{\eta}} } \leq 1$ and $1 - e^{-x} \leq x$.
Thus, by summation,
\begin{equation*}
    q_i^{\eta} \geq \frac{1}{2}-\eta \pi_1 i.
\end{equation*}
Conversely, for all $i \leq I(\eta)$,
\begin{align*}
q_{i}^{\eta}-q_{i+1}^{\eta}
    &= \frac{q_i^{\eta}(1-q_i^\eta) (1- e^{-\eta \pi_1 / q_i^{\eta}})}{1 - q_i^{\eta} + q_i^{\eta} e^{-\eta \pi_1 / q_i^{\eta}})} \\
    &\geq (1-q_i^\eta) q_i^{\eta} (1- e^{-\eta \pi_1 / q_i^{\eta}}) \\
    &\geq (1-q_i^\eta) (\eta \pi_1 - \frac{1}{2} \frac{(\eta \pi_1)^2}{q_i^\eta}) \quad \text{since $1 - e^{-x} \geq x - x^2/2$} \\
    &\geq (1-\frac{1}{2}) (\eta\pi_1 - \frac{\eta\pi_1}{2}) \quad \text{since $\eta \pi_1 \leq q_i^\eta \leq 1/2$ for $i \leq I(\eta)$} \\
    &= \frac{\eta \pi_1}{4}.
\end{align*}
The upper bound on $q_i^\eta$ follows by summation.
For the bracketing $I(\eta)$, by definition of $I(\eta)$ and~\eqref{qilineaire},
\begin{equation*}
\eta \pi_1
    \geq q_{I(\eta)+1}^{\eta}
    \geq \frac{1}{2}-(I(\eta)+1) \eta \pi_1
\end{equation*}
and when $\eta \pi_1 \leq 1/2$,
\begin{equation*}
\eta \pi_1
    \leq q_{I(\eta)}^{\eta}
    \leq 
        \frac{1}{2} - I(\eta) \frac{\eta \pi_1}{4}.
\end{equation*}
Therefore,
\begin{equation*}
\frac{1}{2 \eta \pi_1} - 2
    \leq I(\eta)
    \leq \frac{2}{\eta \pi_1} - 4.
\end{equation*}
\end{proof}

\begin{lemma}
\label{lipschitzianityJ(n,delta)}
Let $\delta$ and $\eta$ such that $\frac{1}{\pi_1}> \eta > \delta > 0$. Then for all $i \geq 0$,
\begin{equation*}
0 \leq q_{i}^\delta - q_{i}^\eta
    \leq (\eta - \delta) \frac{8}{\delta \pi_1}  \frac{(\frac{8}{\delta \pi_1})^i - 1}{\frac{8}{\delta \pi_1} - 1}.
\end{equation*}
\end{lemma}

\begin{proof}
Let 
\begin{equation*}
    g_1(\delta,q) = \frac{q e^{-\delta \pi_1/q}}{1-q + q e^{-\delta \pi_1/q}}.
\end{equation*}
From Equation ~\eqref{partialeta}, for $\frac{1}{\pi_1}>\eta>0$ and $q \in (0,\frac{1}{2})$,
\begin{align*}
\left| \frac{\partial g_1(\eta,q)}{\partial \eta}\right|
    &= \frac{(1-q) \pi_1 e^{-\eta \pi_1/q} }{ (1-q + q e^{-\eta \pi_1/q})^2} \\& \leq \frac{\pi_1 e^{-\eta \pi_1/q}}{ 1-q + q e^{-\eta \pi_1/q}} \\
    &\leq 2 \pi_1 e^{-\eta \pi_1/q},
\end{align*}
since $1-q + q e^{-\eta \pi_k/q}\geq 1-q \geq \frac{1}{2}$.
And from Equation~\eqref{partialq}, for $\frac{1}{\pi_1}>\eta>0$ and $q \in (0,\frac{1}{2})$,
\begin{align*}
\frac{\partial g_1(\eta,q)}{\partial q}
    &= \frac{e^{-\eta \pi_1/q}(1 + \frac{(1-q)\eta \pi_1}{q}) }{(1-q+q e^{-\eta \pi_1/q})^2} \\
    &\leq \frac{e^{-\eta \pi_1/q}}{q} 4 (1+\eta \pi_1) \\
    & \leq 8 \frac{e^{-\eta \pi_1/q}}{q}.
\end{align*}
In particular, since $\pi_1 \leq 1 \leq \frac{1}{q}$, 
\begin{equation*}
    \max \left(\frac{\partial g_1(\eta,q)}{\partial q},\left| \frac{\partial g_1(\eta,q)}{\partial \eta}\right| \right) \leq 8 \frac{e^{-\delta \pi_1/ q}}{q} \leq \frac{8}{\delta \pi_1},
\end{equation*}
where we used $e^{x} \geq x$ applied to $x = \frac{\delta \pi_1}{q}$. Let $\eta > \delta$, then by the mean value theorem, for all $i \geq 0$,
\begin{align*}
    q_{i+1}^\delta - q_{i+1}^\eta &=  q_{i+1}^\delta - g_1( \eta,q_{i}^\delta) + g_1(\eta,q_{i}^\delta)- q_{i+1}^\eta
    \\& \leq 8 \frac{e^{-\delta \pi_1/ q_{i}^{\delta}}}{q_{i}^\delta} \left(q_{i}^\delta- q_{i}^\eta + \eta - \delta \right)
    \\& \leq \frac{8}{\delta \pi_1}(q_i^\delta - q_i^\eta + \eta - \delta).
\end{align*}
Therefore, since $q_0^\delta = q_0^\eta = \frac{1}{2}$, iterating this equation leads to: for all $i \geq 0$,
\begin{align*}
q_{i}^\delta - q_{i}^\eta
    &\leq \sum_{k=0}^{i-1} \left(\frac{8}{\delta \pi_1}\right)^{k+1} (\eta - \delta) \\
    &= (\eta - \delta) \frac{8}{\delta \pi_1} \frac{(\frac{8}{\delta \pi_1})^i - 1}{\frac{8}{\delta \pi_1} - 1}.
\end{align*}
The inequality $q_i^\delta - q_i^\eta \geq 0$ follows from the fact that $\delta \mapsto q_i^\eta$ is non-increasing, see Lemma~\ref{summarypropertiesqi} (ii).
\end{proof}

\subsection{For a decreasing learning rate}

In this section, we prove all the results concerning Section \ref{decreasinglearningratesection}.

\subsubsection{Proof of Proposition~\ref{choiceofTau}: choice of a stopping time.}
\label{proofofpropositionchoiceofTau}

Let $\eta_0 \in \Theta$ and $1 \leq k \leq K$.
\begin{equation*}
p_{k,t+1}^{\eta_n}
    = \frac{p_{k,t}^{\eta_n} e^{-\eta_n \pi_k / p_{k,t}^{\eta_n}}}{(1-p_{k,t}^{\eta_n})+p_{k,t}^{\eta_n} e^{-\eta_n \pi_k / p_{k,t}^{\eta_n}}} \one_{I_t = k}
    + \sum_{\underset{j \neq k}{j=1}}^K \frac{p_{k,t}^{\eta_n}}{(1-p_{j,t}^{\eta_n}) + p_{j,t}^{\eta_n}e^{-\eta_n \pi_j / p_{j,t}^{\eta_n}}} \one_{I_t=j}.
\end{equation*}
For any $q \in [0,1]$, $1-q + q e^{-\eta \pi_k/q} \leq 1$. Therefore, 
\begin{equation*}
p_{k,t+1}^{\eta_n}
    \geq p_{k,t}^{\eta_n} e^{-\eta_n \pi_k / p_{k,t}^{\eta_n}} \one_{I_t = k}
    + \sum_{\underset{j \neq k}{j=1}}^K p_{k,t}^{\eta_n} \one_{I_t=j}.
\end{equation*}
Since $e^{-\eta_n \pi_k / p_{k,t}^{\eta_n}} \leq 1$, 
\begin{equation*}
p_{k,t+1}^{\eta_n}
    \geq p_{k,t}^{\eta_n} e^{-\eta_n \pi_k / p_{k,t}^{\eta_n}} \one_{I_t = k}
    + \sum_{\underset{j \neq k}{j=1}}^K p_{k,t}^{\eta_n} e^{-\eta_n \pi_k / p_{k,t}^{\eta_n}} \one_{I_t=j}
    = p_{k,t}^{\eta_n} e^{-\eta_n \pi_k / p_{k,t}^{\eta_n}}
    \geq p_{k,t}^{\eta_n} (1-\frac{\eta_n \pi_k}{p_{k,t}^{\eta_n}}) = p_{k,t}^{\eta_n}- \eta_n \pi_k.
\end{equation*}
Summing from 1 to t, since $p_{k,1}^{\eta_n}= \frac{1}{K}$,
\begin{equation*}
p_{k,t}^{\eta_n} \geq \frac{1}{K} - \eta_n \pi_k t.
\end{equation*}
Hence, choosing $\displaystyle \Upsilon_n = \Big \lfloor (\frac{1}{K}- \varepsilon) \frac{n^\alpha}{R}\Big \rfloor$ implies that for all $\eta_0 \in \Theta$, $t \leq \Upsilon_n$ and $1 \leq k \leq K$,
\begin{equation*}
\varepsilon
    \leq \frac{1}{K} - \frac{R}{n^\alpha \pi_1} \pi_1 \Upsilon_n
    = \frac{1}{K}- R_n \pi_1 \Upsilon_n \leq \frac{1}{K}- \eta_n \pi_1 \Upsilon_n
    \leq \frac{1}{K}- \eta_n \pi_k \Upsilon_n
    \leq \frac{1}{K}- \eta_n \pi_k t \leq  p_{k,t}^{\eta_n}.
\end{equation*}

\subsubsection{Proof of Lemma \ref{lem_ptlipschitz}: lipschitz character of the updated probability.}\label{proofoflem_ptlipschitz}

Recall that, for $\eta_0 \in \Theta$, for all $1 \leq k \leq K$, for all $t \leq \Upsilon_n-1$,
\begin{equation*}
p_{k,t+1}^{\eta_n}
    = \frac{p_{k,t}^{\eta_n} e^{-\eta_n \pi_k / p_{k,t}^{\eta_n}}}{(1-p_{k,t}^{\eta_n})+p_{k,t}^{\eta_n} e^{-\eta_n \pi_k / p_{k,t}^{\eta_n}}} \one_{I_t = k}
    + \sum_{\underset{j \neq k}{j=1}}^K \frac{p_{k,t}^{\eta_n}}{(1-p_{j,t}^{\eta_n}) + p_{j,t}^{\eta_n}e^{-\eta_n \pi_j / p_{j,t}^{\eta_n}}} \one_{I_t=j}.
\end{equation*}
For all $1 \leq k \leq K$, denote by $h_k$ and $g_k$ the following functions:
\begin{equation*}
h_k(\eta,q) := \frac{1}{1-q + q e^{-\eta \pi_k/q}}
    \quad \text{and} \quad
    g_k(\eta,q) = q e^{-\eta \pi_k / q} h_k(\eta,q).
\end{equation*}
Then,
\begin{equation*}
p_{k,t+1}^{\eta_n}
    = g_k(\eta_n,p_{k,t}^{\eta_n}) \one_{I_t = k}+ \sum_{\underset{j \neq k}{j=1}}^K p_{k,t}^{\eta_n}h_j(\eta_n,p_{j,t}^{\eta_n}) \one_{I_t=j}.
\end{equation*}
Let $\delta_0, \delta_0' \in \Theta.$ Let $1 \leq k \leq K$. Then, for all $t \leq \Upsilon_n-1$,
\begin{equation}
\label{ecart}
|p_{k,t+1}^{\delta_n}- p_{k,t+1}^{\delta_n'}|
    \leq | g_k(\delta_n,p_{k,t}^{\delta_n})- g_k(\delta_n',p_{k,t}^{\delta_n'}) | \one_{I_t = k}
    + \sum_{\underset{j \neq k}{j=1}}^K |p_{k,t}^{\delta_n}h_j(\delta_n,p_{j,t}^{\delta_n})-p_{k,t}^{\delta_n'}h_j(\delta_n',p_{j,t}^{\delta_n'})| \one_{I_t=j}.
\end{equation}

The following lemma controls the derivatives of $h_k$ and $g_k$ and is proved at the end of the current section. Recall that $\pi_1 \geq \pi_k \geq 0$ for all $k$.
\begin{lemma}
\label{Lipschitzianity}
Let $k \in \{1, \dots, K\}$.
For all $\eta > 0$ such that $\eta \pi_1 < 1$,
\begin{equation*}
\begin{cases}
    0 \leq h_k(\eta,q) \leq \frac{1}{1 - \eta \pi_1} \\
    0 \leq \frac{\partial h_k}{\partial q}(\eta,q) \leq \frac{1}{(1 - \eta \pi_1)^2} (\frac{\eta \pi_1}{q} \wedge 1)^2 \\
    0 \leq \frac{\partial h_k}{\partial \eta}(\eta,q) \leq \frac{\pi_1}{1 - \eta \pi_1} \\
    - \pi_1 \leq \frac{\partial g_k}{\partial \eta}(\eta,q) \leq 0 \\
    0 \leq \frac{\partial g_k}{\partial q}(\eta,q) \leq \frac{1}{(1 - \eta \pi_1)^2}.
\end{cases}
\end{equation*}
\end{lemma}

Let us now consider the two cases.
Let $n_0$ be such that $R n_0^{-\alpha} \leq \varepsilon^2$, so that $\eta_n \pi_1 \leq \varepsilon^2 \leq \varepsilon$ for all $n \geq n_0$ and $\eta_0 \in \Theta$, and let $R_n = \frac{R}{n^\alpha \pi_1}$.
Recall that $\Upsilon_n = (\frac{1}{K} - \varepsilon) \frac{n^\alpha}{R}$ is chosen such that $p_{k,t}^{\eta_n} \geq \varepsilon$ for all $k$, $t \leq \Upsilon_n$ and $\eta_0 \in \Theta$ with Proposition~\ref{choiceofTau}.
\begin{itemize}
\item If $I_t = j \neq k$, then, by the triangle inequality,
\begin{align*}
    |p_{k,t+1}^{\delta_n}- p_{k,t+1}^{\delta_n'}| 
    &\leq |p_{k,t}^{\delta_n}h_j(\delta_n,p_{j,t}^{\delta_n})-p_{k,t}^{\delta_n}h_j(\delta_n',p_{j,t}^{\delta_n'})|
    + |p_{k,t}^{\delta_n}h_j(\delta_n',p_{j,t}^{\delta_n'})-p_{k,t}^{\delta_n'}h_j(\delta_n',p_{j,t}^{\delta_n'})| \\
    &\leq |h_j(\delta_n,p_{j,t}^{\delta_n})-h_j(\delta_n',p_{j,t}^{\delta_n'})|
    + h_j(\delta_n',p_{j,t}^{\delta_n'}) |p_{k,t}^{\delta_n}- p_{k,t}^{\delta_n'}| \\ 
    &\leq |h_j(\delta_n,p_{j,t}^{\delta_n})-h_j(\delta_n,p_{j,t}^{\delta_n'})|
    + |h_j(\delta_n,p_{j,t}^{\delta_n'})-h_j(\delta_n',p_{j,t}^{\delta_n'})|
    + h_j(\delta_n',p_{j,t}^{\delta_n'}) \|p_{t}^{\delta_n}- p_{t}^{\delta_n'}\|_{\infty}.
\end{align*}

Then, by Lemma~\ref{Lipschitzianity} and the mean value theorem, for all $n \geq n_0$,
\begin{align}
\nonumber
|p_{k,t+1}^{\delta_n}- p_{k,t+1}^{\delta_n'}|
\label{majorationpourkneqj}
    &\leq \left( \frac{1}{1 - R_n \pi_1} + \frac{(R_n \pi_1)^2}{\varepsilon^2 (1 - R_n \pi_1)^2} \right) \|p_{t}^{\delta_n}- p_{t}^{\delta_n'}\|_{\infty}
        + \frac{\pi_1}{1 - R_n \pi_1} |\delta_n-\delta_n'|.
\end{align}

\item If $I_t = k$, then, by the triangle inequality, Lemma~\ref{Lipschitzianity} and the mean value theorem, for all $n \geq n_0$,
\begin{align*}
|p_{k,t+1}^{\delta_n}- p_{k,t+1}^{\delta_n'}|
    &\leq | g_k(\delta_n,p_{k,t}^{\delta_n}) - g_k(\delta_n',p_{k,t}^{\delta_n}) |
    + | g_k(\delta_n',p_{k,t}^{\delta_n}) - g_k(\delta_n',p_{k,t}^{\delta_n'}) | \\
    &\leq \pi_1 |\delta_n-\delta_n'| 
    + \frac{1}{(1 - R_n \pi_1)^2} |p_{k,t}^{\delta_n'}-p_{k,t}^{\delta_n}| \\
    &\leq \pi_1 |\delta_n-\delta_n'|
    + \frac{1}{(1 - R_n \pi_1)^2} \|p_{t}^{\delta_n'}-p_{t}^{\delta_n}\|_{\infty}.
\end{align*}
\end{itemize}

Thus, almost surely, for all $n \geq n_0$ and $t \leq \Upsilon_n-1$,
\begin{equation*}
\|p_{t+1}^{\delta_n'}-p_{t+1}^{\delta_n}\|_{\infty}
    \leq \frac{1 - R_n \pi_1 + (\frac{R_n \pi_1}{\varepsilon})^2}{(1 - R_n \pi_1)^2} \|p_{t}^{\delta_n'}-p_{t}^{\delta_n}\|_{\infty}
    + \frac{\pi_1}{1 - R_n \pi_1} |\delta_n - \delta_n'|.
\end{equation*}
Now, $\|p_{1}^{\delta_n'}-p_{1}^{\delta_n}\|_\infty = 0$, so that
\begin{equation*}
\|p_{t}^{\delta_n'}-p_{t}^{\delta_n}\|_{\infty}
    \leq \frac{\pi_1}{1 - R_n \pi_1} |\delta_n - \delta_n'| \sum_{s=1}^{t-1} \left(\frac{1 - R_n \pi_1 + (\frac{R_n \pi_1}{\varepsilon})^2}{(1 - R_n \pi_1)^2}\right)^{s-1}.
\end{equation*}
The latter quantity is an increasing function of $t$, thus for all $t \leq \Upsilon_n$,
\begin{equation*}
\|p_{t}^{\delta_n'}-p_{t}^{\delta_n}\|_{\infty}
    \leq \frac{\pi_1}{1 - R_n \pi_1} |\delta_n - \delta_n'| \frac{\left(\frac{1 - R_n \pi_1 + (\frac{R_n \pi_1}{\varepsilon})^2}{(1 - R_n \pi_1)^2}\right)^{\Upsilon_n-1} - 1}{\frac{1 - R_n \pi_1 + (\frac{R_n \pi_1}{\varepsilon})^2}{(1 - R_n \pi_1)^2} - 1}.
\end{equation*}
Now,
\begin{equation*}
\frac{1 - R_n \pi_1 + (\frac{R_n \pi_1}{\varepsilon})^2}{(1 - R_n \pi_1)^2}
    = 1 + R_n \pi_1 + u.
\end{equation*}
where $0 \leq u \leq c (R_n \pi_1)^2$ for some numerical constant $c$, for instance $c=\frac{76}{9}$, since $R_n \pi_1 \leq \varepsilon^2 \leq \frac{1}{4}$.
Therefore,
\begin{equation*}
\left(\frac{1 - R_n \pi_1 + (\frac{R_n \pi_1}{\varepsilon})^2}{(1 - R_n \pi_1)^2}\right)^{\Upsilon_n-1}
    = \exp((\Upsilon_n-1) \log(1 + R_n \pi_1 + u))
    = \exp(\Upsilon_n R_n \pi_1 + \tilde{u})
\end{equation*}
where $|\tilde{u}| \leq R_n \pi_1 + \tilde{c} \Upsilon_n (R_n \pi_1)^2$ for some numerical constant $\tilde{c}$, for instance $\tilde{c} = 22$.
Recall that $\Upsilon_n = (\frac{1}{K}-\varepsilon)\displaystyle \frac{n^\alpha}{R} = \frac{\frac{1}{K}-\varepsilon}{R_n \pi_1}$, so that $|\tilde{u}| \leq 12 R_n \pi_1 \leq 3$ and
\begin{equation*}
\left(\frac{1 - R_n \pi_1 + (\frac{R_n \pi_1}{\varepsilon})^2}{(1 - R_n \pi_1)^2}\right)^{\Upsilon_n-1}
    = \exp (\frac{1}{K}-\varepsilon) \exp(\tilde{u}).
\end{equation*}
All in all,
\begin{align*}
\|p_{t}^{\delta_n'}-p_{t}^{\delta_n}\|_{\infty}
    &\leq \frac{\pi_1 |\delta_n - \delta_n'|}{R_n \pi_1 + u} \left( \exp( \frac{1}{K} - \varepsilon ) \exp(\tilde{u}) - 1 \right) \\
    &\leq c' \frac{|\delta_0 - \delta_0'|}{R}
\end{align*}
for some numerical constant $c'$, for instance $c' = 11$.

\begin{proof}[Proof of Lemma~\ref{Lipschitzianity}]
First of all, using $e^{-x} \geq 1-x$ and $1 - \eta \pi_k \geq 1 - \eta \pi_1 > 0$,
\begin{equation*}
h_k(\eta,q) \leq \frac{1}{1-\eta \pi_k} \leq \frac{1}{1-\eta \pi_1}.
\end{equation*}
Let's study the derivatives of $h_k$:
\begin{equation*}
\frac{\partial h_k}{\partial q}(\eta,q) = \frac{1-e^{-\eta \pi_k/q}(1+\frac{\eta \pi_k}{q})}{(1-q + q e^{-\eta \pi_k/q})^2} \geq 0
\end{equation*}
since $e^{-x}(1+x) \leq 1$, and
\begin{equation*}
\frac{\partial h_k}{\partial q}(\eta,q)
    \leq 
\begin{cases}
    \frac{1 - (1 - \frac{\eta \pi_k}{q})(1 + \frac{\eta \pi_k}{q})}{(1-q + q e^{-\eta \pi_k/q})^2}
    \leq \frac{(\frac{\eta \pi_1}{q})^2}{(1-\eta \pi_1)^2} \quad \text{if } \eta \pi_k/q \leq 1 \\
    \frac{1}{(1-q + q e^{-\eta \pi_k/q})^2}
    \leq \frac{1}{(1-\eta \pi_1)^2} \quad \text{if } \eta \pi_k/q > 1
\end{cases}
\end{equation*}
Similarly, since $1 - q + q e^{-x} \geq (1-q)e^{-x} + qe^{-x} = e^{-x}$ for any $x \geq 0$,
\begin{equation*}
\frac{\partial h_k}{\partial \eta}(\eta,q)
    = \frac{\pi_k e^{-\eta \pi_k/q} }{(1-q + q e^{-\eta \pi_k/q})^2}
    \leq \frac{\pi_k}{1-\eta \pi_k}
    \leq \frac{\pi_1}{1-\eta \pi_1}.
\end{equation*}
Now let's study the derivatives of $g_k$.
For all $q \in (\varepsilon,1]$,
\begin{equation}
\label{partialeta}
\frac{\partial g_k(\eta,q)}{\partial \eta}
    = \frac{-(1-q) \pi_k e^{-\eta \pi_k/q} }{ (1-q + q e^{-\eta \pi_k/q})^2} \leq 0.
\end{equation}
Thus,
\begin{equation*}
-\frac{\partial g_k(\eta,q)}{\partial \eta}
    \leq \frac{(1-q) \pi_k}{1-q +q e^{-\eta \pi_k/q}}
    \leq \pi_k
    \leq \pi_1,
\end{equation*}
and
\begin{align}
\nonumber
\frac{\partial g_k(\eta,q)}{\partial q}
    &= e^{-\eta \pi_k/q} h_k(\eta,q) + q \frac{\eta \pi_k}{q^2} e^{-\eta \pi_k/q} h_k(\eta,q) + q e^{-\eta \pi_k/q} \frac{\partial h_k(\eta,q)}{\partial q} \\
\nonumber
    &= \frac{e^{-\eta \pi_k/q}}{(1-q+q e^{-\eta \pi_k/q})^2} \left( (1 + \frac{\eta \pi_k}{q})(1-q+q e^{-\eta \pi_k/q}) + q(1 - e^{-\eta \pi_k/q}(1 + \frac{\eta \pi_k}{q})) \right)
    \\
\label{partialq}
    &= \frac{e^{-\eta \pi_k/q}(1 + \frac{(1-q)\eta \pi_k}{q}) }{(1-q+q e^{-\eta \pi_k/q})^2} \geq 0.
\end{align}
Now, since $e^x \geq 1 + x$, $1 + \frac{(1-q)\eta \pi_k}{q} \leq 1+ \frac{\eta \pi_k}{q} \leq e^{\eta \pi_k /q}$, so
\begin{equation*}
\frac{\partial g_k(\eta,q)}{\partial q}
    \leq \frac{1}{(1-q+q e^{-\eta \pi_k/q})^2}
    \leq \frac{1}{(1-\eta \pi_k)^2}
    \leq \frac{1}{(1-\eta \pi_1)^2}.
\end{equation*}
\end{proof}

\subsubsection{Proof of Proposition~\ref{hypothesebaraud}} 
\label{proofofbaraud}

For $\delta_0 \in \Theta$, recall that $\displaystyle X_{\delta_n} = \ell_{n,\varepsilon}(\delta_0) - \sum_{t=1}^{\Upsilon_n} \sum_{k=1}^K \log(p_{k,t}^{\delta_n}) p_{k,t}^{\eta_n}$. For $t \geq 1$, let
\begin{equation*}
Y_t^{\delta_n}
    := \sum_{k=1}^K \log(p_{k,t}^{\delta_n}) \one_{I_t = k}
\end{equation*}
and let
\begin{equation*}
C_{t}^{\delta_n}
    = \EE_{\eta_n}[Y_t^{\delta_n}|\F_{t}]
    = \sum_{k=1}^K \log(p_{k,t}^{\delta_n}) p_{k,t}^{\eta_n}
\end{equation*}
be its compensator. Then, $X_{\delta_n} = \sum_{t=1}^{\Upsilon_n} (Y_t^{\delta_n}-C_t^{\delta_n})$. Since $\widehat{\eta}_0$ is the maximum likelihood estimator, the following inequalities hold:
\begin{align*}
\sum_{t=1}^{\Upsilon_n} (C_t^{\eta_n}-C_t^{\widehat{\eta}_n})
    &\leq \underbrace{\ell_{n,\varepsilon}(\widehat{\eta}_0) - \ell_{n,\varepsilon}(\eta_0)}_{\geq 0}
    - \sum_{t=1}^{\Upsilon_n} (C_t^{\widehat{\eta}_n}-C_t^{\eta_n}) \\
    &= \sum_{t=1}^{\Upsilon_n} (Y_t^{\widehat{\eta}_n}-Y_t^{\eta_n}-(C_t^{\widehat{\eta}_n}-C_t^{\eta_n})) \\
    &\leq \sup_{\delta_0 \in \Theta} \sum_{t=1}^{\Upsilon_n} (Y_t^{\delta_n}-Y_t^{\eta_n}-(C_t^{\delta_n}-C_t^{\eta_n})).
\end{align*}
Let $\delta_0, \delta_0' \in \Theta$. Let $Z_1 = 0$ and for $2 \leq t \leq \Upsilon_n$, let $Z_t = \displaystyle \sum_{s=1}^{t-1}(Y_s^{\delta_n} -Y_s^{\delta_n'}-(C_s^{\delta_n}-C_s^{\delta_n'})) $. Then, $(Z_t)_t$ is an $(\F_t)_t$-martingale. Indeed,
\begin{equation*}
\EE_{\eta_n}[Z_{t+1}|\F_t]
    = \underbrace{\EE_{\eta_n}[Y_t^{\delta_n}-Y_{t}^{\delta_n'}|\F_t] - (C_t^{\delta_n}-C_{t}^{\delta_n'})}_{=0}
    + \sum_{s=1}^{t-1} (Y_s^{\delta_n}-Y_s^{\delta_n'}-(C_s^{\delta_n}-C_s^{\delta_n'}))
    = Z_t.
\end{equation*}
Let $A_t^k = \sum_{s=2}^t \EE_{\eta_n}[(Z_s-Z_{s-1})^k|\F_{s-1}]$ if $t \geq 2$ and 0 if $t=1$. Let $\lambda \geq 0$. From Lemma~3.3 in \cite{Houdre}, the sequence $(\Ecal_t)_{t \geq 1}$ defined for all integers $t \geq 1$ by
\begin{equation*}
\Ecal_t = \exp (\lambda Z_t - \sum_{k \geq 2} \frac{\lambda^k}{k!} A_t^k )
\end{equation*}
is a supermartingale with respect to the filtration $(\F_t)_t$. As a result, for any $t \geq 1$,
\begin{equation*}
\EE_{\eta_n}[\Ecal_t]
    \leq \EE_{\eta_n}[ \Ecal_1]
    = 1.
\end{equation*}
Therefore, for any $t \geq 1$,
\begin{equation}
\label{supermartingale}
\EE_{\eta_n}[\exp (\lambda Z_t)]
    \leq \EE_{\eta_n} [ \exp ( \sum_{k \geq 2} \frac{\lambda^k}{k!} A_t^k ) ]. 
\end{equation}

Let $t \geq 2$. When there exists a constant $C_Z$ such that $|Z_s-Z_{s-1}| \leq C_Z$ for all $2 \leq s \leq t$ almost surely, it holds for all $k \geq 2$,
\begin{equation}
\label{separationnormes}
|A_t^k|
    \leq \sum_{s=2}^t \EE_{\eta_n}[(Z_s-Z_{s-1})^{2}|\F_{s-1}] C_Z^{k-2}.
\end{equation}
For all $u \geq 2$,
\begin{equation*}
Z_u - Z_{u-1}
    = \sum_{k=1}^K \log (\frac{p_{k,u}^{\delta_n}}{p_{k,u}^{\delta_n'}})(\one_{I_u=k}-p_{k,u}^{\eta_n}).
\end{equation*}
Hence,
\begin{equation*}
|Z_u - Z_{u-1}|
    \leq \sum_{k=1}^K |\log (\frac{p_{k,u}^{\delta_n}}{p_{k,u}^{\delta_n'}})| \cdot |\one_{I_u=k}-p_{k,u}^{\eta_n}|
    \leq \left(\sup_{1 \leq k \leq K} |\log (\frac{p_{k,u}^{\delta_n}}{p_{k,u}^{\delta_n'}})|\right)
    \sum_{k=1}^K |\one_{I_u=k}-p_{k,u}^{\eta_n}|.
\end{equation*}
Then, using the lipschitzianity of the $\log$ function on $[\varepsilon,1]$ and Lemma~\ref{lem_ptlipschitz},
\begin{equation*}
|Z_u - Z_{u-1}|
    \leq \frac{2c}{R \varepsilon} |\delta_0-\delta_0'|.
\end{equation*}
Therefore, Equation~\eqref{separationnormes} becomes 
\begin{equation}
\label{majorationuniforme}
|A_t^k|
    \leq (\frac{2c}{R \varepsilon} |\delta_0-\delta_0'|)^{k-2} \sum_{s=2}^t \EE_{\eta_n}[(Z_s-Z_{s-1})^{2}|\F_{s-1}].
\end{equation}
Let us control the order 2 moment in the above equation:
\begin{align*}
\sum_{s=2}^t \EE_{\eta_n}[(Z_s-Z_{s-1})^{2}|\F_{s-1}]
    &= \sum_{s=2}^{t} \EE_{\eta_n} [(Y_s^{\delta_n}-Y_s^{\eta_n}-(C_s^{\delta_n}-C_s^{\eta_n}))^2 | \F_{s-1}] \\
    &= \sum_{s=2}^{t} \EE_{\eta_n} \left[\left(\sum_{k=1}^K \log (\frac{p_{k,s}^{\delta_n}}{p_{k,s}^{\delta_n'}})(\one_{I_t=k}-p_{k,s}^{\eta_n})\right)^2 | \F_{s-1}\right] \\
    &= \sum_{s=2}^{t} \EE_{\eta_n} \left[\sum_{k=1}^K \log (\frac{p_{k,s}^{\delta_n}}{p_{k,s}^{\delta_n'}})^2(\one_{I_s=k}-p_{k,s}^{\eta_n})^2 | \F_{s-1}\right] \\
    &\quad + \sum_{s=2}^{t} \EE_{\eta_n} \left[\sum_{j \neq k} \log (\frac{p_{k,s}^{\delta_n}}{p_{k,s}^{\delta_n'}})\log (\frac{p_{j,s}^{\delta_n}}{p_{j,s}^{\delta_n'}})(\one_{I_s=k}-p_{k,s}^{\eta_n})(\one_{I_s=j}-p_{k,s}^{\eta_n})) |\F_{s-1}\right] \\
    &= \sum_{s=2}^{t} \sum_{k=1}^K \log (\frac{p_{k,s}^{\delta_n}}{p_{k,s}^{\delta_n'}})^2(1-p_{k,s}^{\eta_n})p_{k,s}^{\eta_n}
    - \sum_{s=2}^{t} \sum_{j \neq k} \log (\frac{p_{k,s}^{\delta_n}}{p_{k,s}^{\delta_n'}}) \log (\frac{p_{j,s}^{\delta_n}}{p_{j,s}^{\delta_n'}}) p_{k,s}^{\eta_n} p_{j,s}^{\eta_n},
\end{align*}
where for the last line we used that
\begin{equation*}
\EE_{\eta_n}[(\one_{I_s=k}-p_{k,s}^{\eta_n})^2|\F_{s-1}]
    = \VV_{\eta_n}[\one_{I_s=k}|\F_{s-1}] = (1-p_{k,s}^{\eta_n})p_{k,s}^{\eta_n}
\end{equation*}
and
\begin{equation*}
\EE_{\eta_n}[(\one_{I_s=k}-p_{k,s}^{\eta_n})(\one_{I_s = \ell}-p_{\ell,s}^{\eta_n})|\F_{s-1}]
    = - p_{\ell,s}^{\eta_n} p_{k,s}^{\eta_n}.
\end{equation*}
Therefore,
\begin{align*}
\sum_{s=2}^t \EE_{\eta_n}[(Z_s-Z_{s-1})^{2}|\F_{s-1}]
    &= \sum_{s=2}^{t} \sum_{k=1}^K (\log (\frac{p_{k,s}^{\delta_n}}{p_{k,s}^{\delta_n'}}))^2 p_{k,s}^{\eta_n}
    - \sum_{s=2}^{t} \sum_{j, k} \log (\frac{p_{k,s}^{\delta_n}}{p_{k,s}^{\delta_n'}})\log (\frac{p_{j,s}^{\delta_n}}{p_{j,s}^{\delta_n'}})p_{k,s}^{\eta_n}p_{j,s}^{\eta_n} \\
    &= \sum_{s=2}^{t} \left(\sum_{k=1}^K(\log (\frac{p_{k,s}^{\delta_n}}{p_{k,s}^{\delta_n'}}))^2 p_{k,s}^{\eta_n}
    - \left(\sum_{k=1}^K \log (\frac{p_{k,s}^{\delta_n}}{p_{k,s}^{\delta_n'}}) p_{k,s}^{\eta_n}\right)^2 \right) \\
    &\leq \sum_{s=2}^{t} \sum_{k=1}^K(\log (\frac{p_{k,s}^{\delta_n}}{p_{k,s}^{\delta_n'}}))^2 p_{k,s}^{\eta_n}.
\end{align*}
The $\log$ function is lipschitz on $[\varepsilon,1]$, and its lipschitz constant is $\frac{1}{\varepsilon}$. Therefore,
\begin{equation*}
\sum_{s=2}^t \EE_{\eta_n}[(Z_s-Z_{s-1})^{2}|\F_{s-1}]
    \leq \frac{1}{\varepsilon^2} \sum_{s=2}^{t}\sum_{k=1}^K( p_{k,s}^{\delta_n}-p_{k,s}^{\delta_n'})^2 p_{k,s}^{\eta_n},
\end{equation*}
and by Lemma~\ref{lem_ptlipschitz},
\begin{align}
\nonumber
\sum_{s=2}^t \EE_{\eta_n}[(Z_s-Z_{s-1})^{2}|\F_{s-1}]
    &\leq \frac{c^2}{R^2 \varepsilon^2} \sum_{s=2}^{t} \sum_{k=1}^K|\delta_0-\delta_0'|^2 p_{k,s}^{\eta_n} \\
\label{majorationducarre}
    &= \frac{c^2}{R^2 \varepsilon^2} |\delta_0-\delta_0'|^2 (t-1).
\end{align}

Injecting~\eqref{majorationducarre} and~\eqref{majorationuniforme} in~\eqref{supermartingale},
\begin{equation*}
\EE_{\eta_n}[\exp (\lambda Z_{\Upsilon_n+1})]
    \leq \EE_{\eta_n} \left[ \exp \left(\lambda^2 \frac{c^2}{R^2 \varepsilon^2} |\delta_0-\delta_0'|^2 \Upsilon_n \sum_{k \geq 2} \frac{\lambda^{k-2}}{k!} ( \frac{2c}{R \varepsilon} |\delta_0-\delta_0'| )^{k-2} \right) \right].
\end{equation*}
For all $k \geq 0$, $(k+2)! \geq 2^{k+1}$, so
\begin{equation*}
\forall |\lambda| < (\frac{c}{R \varepsilon} |\delta_0-\delta_0'|)^{-1}, \quad
    \EE_{\eta_n}[\exp (\lambda Z_{\Upsilon_n+1})]
    \leq \exp \left( \frac{(\frac{\lambda c}{R \varepsilon} |\delta_0-\delta_0'|)^2 \Upsilon_n}{2(1- \displaystyle \frac{\lambda c}{R \varepsilon} |\delta_0-\delta_0'|)}\right).
\end{equation*}
To conclude, note that $Z_{\Upsilon_n+1} = X_{\delta_n}-X_{\delta_n'}$.

\subsubsection{Proof of Theorem~\ref{theorembaraud}}
\label{proofofbernsteintype}

The goal is to apply Theorem 2.1 of \cite{https://doi.org/10.48550/arxiv.0909.1863}. The distance $d(\cdot,\cdot)$ derives from a norm, and for all $\delta_0 \in \Theta$,
\begin{equation*}
d(\delta_0,\eta_0)
    = \frac{c}{R \varepsilon} |\delta_0-\eta_0|
    \leq \frac{c (R-r)}{R \varepsilon}
    \leq \frac{c}{\varepsilon}.
\end{equation*}
Therefore, Assumption 2.2 in \cite{https://doi.org/10.48550/arxiv.0909.1863} is verified for the constants $v = \frac{c}{\varepsilon} \sqrt{\Upsilon_n}$ and $b = \frac{c}{\varepsilon}$, so that
\begin{equation*}
\forall x \geq 0, \quad
    \PP_{\eta_n}\bigg[Z_n \geq \kappa ( v \sqrt{x+1} + b(x+1)) \bigg] \leq e^{-x}
\end{equation*}
with $\kappa=18$.

\subsubsection{Proof of Theorem~\ref{leastsquareerror}}
\label{proofofleastsquareerror}

Take the notations of Section~\ref{proofofbaraud}.
Since $\widehat{\eta}_0$ is the maximum likelihood estimator, the following inequalities hold:
\begin{align*}
\sum_{t=1}^{\Upsilon_n} (C_t^{\eta_n}-C_t^{\widehat{\eta}_n})
    &\leq \underbrace{\ell_{n,\varepsilon}(\widehat{\eta}_0) - \ell_{n,\varepsilon}(\eta_0)}_{\geq 0}
    - \sum_{t=1}^{\Upsilon_n} (C_t^{\widehat{\eta}_n}-C_t^{\eta_n}) \\
    &= \sum_{t=1}^{\Upsilon_n} (Y_t^{\widehat{\eta}_n}-Y_t^{\eta_n}-(C_t^{\widehat{\eta}_n}-C_t^{\eta_n})) \\
    &\leq \sup_{\delta_0 \in \Theta} \sum_{t=1}^{\Upsilon_n} (Y_t^{\delta_n}-Y_t^{\eta_n}-(C_t^{\delta_n}-C_t^{\eta_n})).
\end{align*}
Note that
\begin{equation*}
\sum_{t=1}^{\Upsilon_n} (C_t^{\eta_n}-C_t^{\widehat{\eta}_n})
    = \sum_{t=1}^{\Upsilon_n} \sum_{k=1}^K \log ( \frac{p_{k,t}^{\eta_n}}{p_{k,t}^{\widehat{\eta}_n}} ) p_{k,t}^{\eta_n},
\end{equation*}
so that using Pinsker's inequality,
\begin{equation*}
\sum_{t=1}^{\Upsilon_n} (C_t^{\eta_n}-C_t^{\widehat{\eta}_n})
    \geq 2 \sum_{t=1}^{\Upsilon_n} \left( \sum_{k=1}^K | p_{k,t}^{\eta_n}-p_{k,t}^{\widehat{\eta}_n} | \right)^2.
\end{equation*}
By comparison of the norms,
\begin{equation*}
    = \left(\sum_{k=1}^K | p_{k,t}^{\eta_n}-p_{k,t}^{\widehat{\eta}_n} |\right)^2
    \geq \sum_{k=1}^K | p_{k,t}^{\eta_n}-p_{k,t}^{\widehat{\eta}_n} |^2
    = \| p_{t}^{\eta_n}-p_{t}^{\widehat{\eta}_n} \|_2^2.
\end{equation*}
Thus,
\begin{equation*}
\sum_{t=1}^{\Upsilon_n} (C_t^{\eta_n}-C_t^{\widehat{\eta}_n})
    \geq 2 \sum_{t=1}^{\Upsilon_n} \| p_{t}^{\eta_n}-p_{t}^{\widehat{\eta}_n} \|_2^2.
\end{equation*}
All in all,
\begin{equation*}
2 \sum_{t=1}^{\Upsilon_n} \| p_{t}^{\eta_n}-p_{t}^{\widehat{\eta}_n} \|_2^2
    \leq \sup_{\delta_0 \in \Theta} \sum_{t=1}^{\Upsilon_n} (Y_t^{\delta_n}-Y_t^{\eta_n}-(C_t^{\delta_n}-C_t^{\eta_n}))
    = Z_n.
\end{equation*}
Using Theorem~\ref{theorembaraud} with the same constants,
\begin{equation*}
\forall x \geq 0, \quad
    \PP_{\eta_n}\bigg[ 2 \sum_{t=1}^{\Upsilon_n} \| p_{t}^{\eta_n}-p_{t}^{\widehat{\eta}_n} \|_2^2 \geq \frac{18 c}{\varepsilon} \left( \sqrt{\Upsilon_n}\sqrt{x+1} + x+1) \right)\bigg] \leq e^{-x},
\end{equation*}
that is,
\begin{equation*}
\forall x \geq 0, \quad
    \PP_{\eta_n}\bigg[ \frac{1}{\Upsilon_n} \sum_{t=1}^{\Upsilon_n} \| p_{t}^{\eta_n}-p_{t}^{\widehat{\eta}_n} \|_2^2 \geq  \frac{9 c}{\varepsilon} \left( \sqrt{\frac{x+1}{\Upsilon_n}} + \frac{x+1}{\Upsilon_n} \right)\bigg] \leq e^{-x}.
\end{equation*}

\subsection{Application to a special case}

In this section, we prove all the results of Section \ref{specialcasesection}. Some lemmas are used for the two propositions. They can be found at the end of each section.

\subsubsection{Proof of Proposition~\ref{Lowerboundleastsquare}: a lower bound on the prediction error.}
\label{proofoflowerboundleastsquare}

By Lemma \ref{minorationdifferenceqietadelta}, there exists $D_{\pi_1}$ such that
  \begin{equation*}
        |p_{1,t}^{\widehat{\eta}_n}- p_{1,t}^{\eta_n}|  \geq D_{\pi_1} |\widehat{\eta}_n - \eta_n| e^{-(t-2) R_n \pi_1 / \varepsilon} N_{t-1}.
      \end{equation*}
Taking the square and summing,
\begin{equation*}
\sum_{t=1}^{\Upsilon_n} |p_{1,t}^{\delta_n}- p_{1,t}^{\eta_n}|^2
    \geq D_{\pi_1}^2 |\widehat{\eta}_n - \eta_n|^2 \sum_{t=1}^{\Upsilon_n} e^{-2 (t-2) R_n \pi_1 / \varepsilon} N_{t-1}^2
    \geq D_{\pi_1}^2 |\widehat{\eta}_n - \eta_n|^2  e^{-2 {\Upsilon_n} R_n \pi_1 / \varepsilon} \sum_{t=1}^{\Upsilon_n} N_{t-1}^2.
\end{equation*}
Since ${\Upsilon_n} =(\displaystyle  \frac{1}{2}-\varepsilon) \frac{n^\alpha}{R}$,
\begin{equation*}
\sum_{t=1}^{\Upsilon_n} |p_{1,t}^{\delta_n}- p_{1,t}^{\eta_n}|^2
    \geq D_{\pi_1}^2 |\widehat{\eta}_n - \eta_n|^2  e^{- (1-2\varepsilon)/ \varepsilon} \sum_{t=1}^{\Upsilon_n} N_{t-1}^2.
\end{equation*}
Let $m_{\pi_1,\varepsilon} := D_{\pi_1}^2 e^{- (1-2\varepsilon)/ \varepsilon}$.
Then,
\begin{equation*}
\sum_{t=1}^{\Upsilon_n} |p_{1,t}^{\delta_n}- p_{1,t}^{\eta_n}|^2
    \geq m_{\pi_1,\varepsilon} |\widehat{\eta}_n - \eta_n|^2  \sum_{t=1}^{\Upsilon_n} N_{t-1}^2,
\end{equation*}
which is the desired result.

\begin{lemma}
\label{minorationdifferenceqietadelta}
Recall that $R = \max\{\delta_0, \delta_0 \in \Theta\} = \displaystyle \frac{(\frac{1}{K}-\varepsilon) n^{\alpha}}{\Upsilon_n}$. Let $N_t = \sum_{s=1}^{t}\one_{I_{s}=1}$. Let $\eta_0$, $\delta_0 \in \Theta$. There exists $D_{\pi_1} > 0$, such that for any $t \in \{1,\ldots,{\Upsilon_n}\}$
\begin{equation}
\label{minorationlipschitz}
|p_{1,t+1}^{\delta_n} - p_{1,t+1}^{\eta_n}|
    \geq D_{\pi_1} |\eta_n - \delta_n | \sum_{s=1}^{t}e^{-(t-s) R_n \pi_1 / \varepsilon}\one_{I_{s}=1}
    \geq D_{\pi_1} |\eta_n - \delta_n | e^{-(t-1) R_n \pi_1 / \varepsilon} N_{t}.
\end{equation}
\end{lemma}

\begin{proof}
Suppose $\delta_0 \leq \eta_0$. Then, for all $t \geq 1$,
\begin{equation*}
    p_{1,t}^{\delta_n} \geq p_{1,t}^{\eta_n}.
\end{equation*}
Indeed, the result holds for $t=1$. Assume it is true for $t \in \N$. If $I_t=2$, then
\begin{equation*}
p_{1,t+1}^{\delta_n}
    = p_{1,t}^{\delta_n}
    \geq p_{1,t}^{\eta_n}
    = p_{1,t+1}^{\eta_n}.
\end{equation*}
If $I_t = 1$, then 
\begin{equation*}
p_{1,t+1}^{\delta_n}
    = g(p_{1,t}^{\delta_n},\delta_n)
    \geq g(p_{1,t}^{\delta_n},\eta_n)
    \geq g(p_{1,t}^{\eta_n},\eta_n)
    = p_{1,t+1}^{\eta_n}
\end{equation*}
where $g$ is the function defined in the proof of Lemma~\ref{summarypropertiesqi} and verifies that $q \mapsto g(q,\eta)$ is an increasing function and $\eta \mapsto g(q,\eta)$ is a decreasing function. Hence the result. Let's examine 
\begin{equation*}
p_{1,t+1}^{\delta_n} - p_{1,t+1}^{\eta_n}
    = (p_{1,t}^{\delta_n} - p_{1,t}^{\eta_n}) \one_{I_t =2 }
        + (g(p_{1,t}^{\delta_n},\delta_n) - g(p_{1,t}^{\eta_n},\eta_n)) \one_{I_t = 1}.
\end{equation*}
Recall that $R = \max \Theta$. Using Lemma \ref{lem_ptlipschitz2}, 
\begin{align*}
g(p_{1,t}^{\delta_n},\delta_n) - g(p_{1,t}^{\eta_n},\eta_n)
    &\geq g(p_{1,t}^{\delta_n},\delta_n)-g(p_{1,t}^{\eta_n},\delta_n)+g(p_{1,t}^{\eta_n},\delta_n)- g(p_{1,t}^{\eta_n},\eta_n) \\
    &\geq e^{-R_n \pi_1 / \varepsilon}(p_{1,t}^{\delta_n} -p_{1,t}^{\eta_n}) + D_{\pi_1}(\eta_n-\delta_n).
\end{align*}
Therefore,
\begin{align*}
p_{1,t+1}^{\delta_n} - p_{1,t+1}^{\eta_n}
    &\geq (p_{1,t}^{\delta_n} - p_{1,t}^{\eta_n}) \one_{I_t =2} + \Big(\underbrace{e^{-R_n \pi_1 / \varepsilon}}_{\leq 1}(p_{1,t}^{\delta_n} -p_{1,t}^{\eta_n}) + D_{\pi_1}(\eta_n-\delta_n)\Big) \one_{I_t = 1} \\
    &\geq e^{-R_n \pi_1 / \varepsilon}(p_{1,t}^{\delta_n} - p_{1,t}^{\eta_n}) \one_{I_t =2} + (e^{-R_n \pi_1 / \varepsilon}(p_{1,t}^{\delta_n} - p_{1,t}^{\eta_n}) + D_{\pi_1}(\eta_n-\delta_n)) \one_{I_t = 1} \\
    &\geq e^{-R_n \pi_1 / \varepsilon}(p_{1,t}^{\delta_n} - p_{1,t}^{\eta_n}) + D_{\pi_1}(\eta_n-\delta_n) \one_{I_t = 1}.
\end{align*}
Therefore, iterating,
\begin{equation*}
p_{1,t+1}^{\delta_n}- p_{1,t+1}^{\eta_n}
    \geq D_{\pi_1}(\eta_n - \delta_n) \sum_{s=1}^{t}\one_{I_{s}=1} e^{-(t-s) R_n \pi_1 / \varepsilon}.
\end{equation*}
The proof is the same for $\delta_0 > \eta_0$.
\end{proof}

\begin{lemma}
\label{lem_ptlipschitz2}
For all $\gamma \pi_1 \in (0,1]$ and $q \in [\max(\varepsilon,\gamma \pi_1),\frac{1}{2}]$,
\begin{equation*}
|\frac{\partial g_1(\gamma,q)}{\partial \gamma}| \geq \frac{\pi_1 e^{-1}}{2}
    =: D_{\pi_1}
    \quad \text{and} \quad
    \frac{\partial g_1(\gamma,q)}{\partial q} \geq e^{-\gamma \pi_1/ \varepsilon}.
\end{equation*}
\end{lemma}

\begin{proof}
In~\eqref{partialeta}, for all $\gamma \pi_1 \in (0,1]$ and $q \in [\max(\varepsilon,\gamma \pi_1),\frac{1}{2}]$,  
\begin{align*}
|\frac{\partial g_1(\gamma,q)}{\partial \gamma}|
    &= \frac{(1-q) \pi_1 e^{-\gamma \pi_1/q} }{ ((1-q)+q e^{-\gamma \pi_1/q})^2} \\
    &\geq \frac{\pi_1e^{-1}}{2 ((1-q)+q e^{-\gamma \pi_1/q})^2 } \quad \text{since } q \leq \frac{1}{2} \\
    &\geq \frac{\pi_1e^{-1}}{2} =: D_{\pi_1} \quad \text{ using that } 1-q+q e^{-\gamma \pi_1/q} \leq 1.
\end{align*}
Similarly, using~\eqref{partialq}, for all $q \in[\varepsilon,\frac{1}{2}]$,
\begin{align*}
\frac{\partial g_1(\gamma,q)}{\partial q}
    &= \frac{e^{-\gamma \pi_1/q}(1 + \frac{1-q}{q} \gamma \pi_1) }{((1-q)+q e^{-\gamma \pi_1/q})^2} \\
    &\geq \frac{e^{-\gamma \pi_1/\varepsilon}}{((1-q)+q e^{-\gamma \pi_1/q})^2} \\
    &\geq e^{-\gamma \pi_1/\varepsilon}.
\end{align*}
\end{proof}

\begin{lemma}
\label{sommerlesproba}
Let $\delta_1$, $\delta_2$, $\varepsilon_1$, $\varepsilon_2$ be positive numbers. Let $X$ and $Y$ be two random variables such that $\PP(X \geq \varepsilon_1) \leq \delta_1$ and  $\PP(Y \geq \varepsilon_2) \leq \delta_2$. Then,
\begin{equation*}
    \PP(X+Y \geq \varepsilon_1 + \varepsilon_2) \leq \delta_1 + \delta_2.
\end{equation*}
\end{lemma}

\begin{proof}
It is a direct consequence from the inclusion $\{X+Y \geq \varepsilon_1 + \varepsilon_2\} \subset \{X \geq \varepsilon_1\} \cup \{Y \geq \varepsilon_1\}$ and the union bound.
\end{proof}

\subsubsection{Proof of Proposition~\ref{probabilityestimationerror}: an upper bound on the estimation error.}
\label{proofofprobabilityestimationerror}

Let $A_n : \displaystyle \frac{\Upsilon_n(\Upsilon_n-1)(2\Upsilon_n-1)}{96}$. By Lemma~\ref{minoratiosommedesNt2}, for all $y>0$,
\begin{equation}
\label{minorationsommedesNt2}
    \PP_{\eta_n} \left( A_n- \Big(\frac{2}{5} \sqrt{2 \log \frac{2 \Upsilon_n}{y}}  \Upsilon_n^{\frac{5}{2}} + \Upsilon_n^2 \log \frac{2 \Upsilon_n}{y}\Big) \geq \sum_{t=1}^{\Upsilon_n} N_{t-1}^2  \right)
    \leq y . 
\end{equation}
By Proposition~\ref{leastsquareerror},
for all $n \geq (R / \varepsilon^2)^{1/\alpha}$ and $x \geq 0$,,
\begin{equation}
\label{highppleastsquare}
\PP_{\eta_n} \left(\sum_{t=1}^{\Upsilon_n} \| p_{t}^{\eta_n}-p_{t}^{\widehat{\eta}_n} \|_2^2
    < \frac{9 c}{\varepsilon} \left( \sqrt{(x+1)\Upsilon_n} + x+1 \right)\right) \geq 1 - e^{-x}.
\end{equation}
By Proposition~\ref{Lowerboundleastsquare}, there exists a constant $m_{\pi_1,\varepsilon} > 0$, such that for all $\eta_0, \widehat{\eta}_0 \in \Theta$,
\begin{equation*}
    \sum_{t=1}^{\Upsilon_n}  |p_{1,t}^{\widehat{\eta}_n}- p_{1,t}^{\eta_n}|^2 \geq  m_{\pi_1,\varepsilon} |\widehat{\eta}_n - \eta_n|^2 \sum_{t=1}^{\Upsilon_n} N_{t-1}^2.
\end{equation*}
Since $\displaystyle  \sum_{t=1}^{\Upsilon_n}  |p_{1,t}^{\widehat{\eta}_n}- p_{1,t}^{\eta_n}|^2 = \frac{1}{2} \sum_{t=1}^{\Upsilon_n} \| p_{t}^{\eta_n}-p_{t}^{\widehat{\eta}_n} \|_2^2$, injecting this result in Equation~\eqref{highppleastsquare},
\begin{equation*}
\PP_{\eta_n} \left( 2 m_{\pi_1,\varepsilon} |\widehat{\eta}_n - \eta_n|^2 \sum_{t=1}^{\Upsilon_n} N_{t-1}^2
    > \frac{9 c}{\varepsilon} \left( \sqrt{(x+1)\Upsilon_n} + x+1 \right)\right) \leq e^{-x},
\end{equation*}
that is,
\begin{equation}
\label{ppNt}
\PP_{\eta_n} \left( |\widehat{\eta}_n - \eta_n|^2 \sum_{t=1}^{\Upsilon_n} N_{t-1}^2
    > \frac{9 c}{2 m_{\pi_1,\varepsilon} \varepsilon} \left( \sqrt{(x+1)\Upsilon_n} + x+1 \right)\right) \leq e^{-x}.
\end{equation}
Combining \eqref{minorationsommedesNt2} and \eqref{ppNt} with Lemma~\ref{sommerlesproba},
\begin{align*}
& \PP_{\eta_n} \left(  |\widehat{\eta}_n - \eta_n|^2 A_n
    > \frac{9 c}{2 m_{\pi_1,\varepsilon} \varepsilon} \left( \sqrt{(x+1)\Upsilon_n} + x+1 \right)+ |\widehat{\eta}_n - \eta_n|^2 \Big(\frac{2}{5} \sqrt{2 \log \frac{2 \Upsilon_n}{y}}  \Upsilon_n^{\frac{5}{2}} + \Upsilon_n^2 \log \frac{2 \Upsilon_n}{y}\Big)\right)
    \\& \leq \PP_{\eta_n} \left(  |\widehat{\eta}_n - \eta_n|^2 \sum_{t=1}^{\Upsilon_n} N_{t-1}^2
    > \frac{9 c}{2 m_{\pi_1,\varepsilon} \varepsilon} \left( \sqrt{(x+1)\Upsilon_n} + x+1 \right)\right) \\& + \PP_{\eta_n} \left( A_n- \sum_{t=1}^{\Upsilon_n} N_{t-1}^2 \geq \Big(\frac{2}{5} \sqrt{2 \log \frac{2 \Upsilon_n}{y}}  \Upsilon_n^{\frac{5}{2}} + \Upsilon_n^2 \log \frac{2 \Upsilon_n}{y}\Big) \right)
    \\& \leq e^{-x}+y.
\end{align*}
Therefore, with probability at least $1-(e^{-x}+y)$,
\begin{equation*}
    |\widehat{\eta}_n - \eta_n|^2 \left(A_n-\Big(\frac{2}{5} \sqrt{2 \log \frac{2 \Upsilon_n}{y}}  \Upsilon_n^{\frac{5}{2}} + \Upsilon_n^2 \log \frac{2 \Upsilon_n}{y}\Big)\right)
    \leq \frac{9 c}{2 m_{\pi_1,\varepsilon} \varepsilon} \left( \sqrt{(x+1)\Upsilon_n} + x+1 \right).
\end{equation*}
Recall that $\displaystyle \Upsilon_n = \big(\frac{1}{2}-\varepsilon \big)\frac{n^\alpha}{R}$ and that
\begin{equation*}
    |\widehat{\eta}_n - \eta_n|
    = \frac{1}{n^{\alpha} \pi_1} |\widehat{\eta}_0-\eta_0|
    = \frac{\frac{1}{2}-\varepsilon}{R \pi_1} |\widehat{\eta}_0-\eta_0| \frac{1}{\Upsilon_n }.
\end{equation*}
Let $B_n := \displaystyle \frac{A_n}{\Upsilon_n^2}$. Then, with probability at least $1-(e^{-x}+y)$,
\begin{equation*}
    |\widehat{\eta}_0 - \eta_0|^2 \left(B_n-\Big(\frac{2}{5} \sqrt{2 \log \frac{2 \Upsilon_n}{y}}  \Upsilon_n^{\frac{1}{2}} +  \log \frac{2 \Upsilon_n}{y}\Big)\right)
    \leq \frac{9 c (R \pi_1)^2}{2(\frac{1}{2}-\varepsilon)^2 m_{\pi_1,\varepsilon} \varepsilon} \left( \sqrt{(x+1)\Upsilon_n} + x+1 \right),
\end{equation*}
that is,
\begin{equation*}
     |\widehat{\eta}_0 - \eta_0| \leq \frac{R \pi_1}{\frac{1}{2}-\varepsilon}\sqrt{\frac{9 c}{ 2 m_{\pi_1,\varepsilon} \varepsilon}} \left(\frac{\sqrt{(x+1)\Upsilon_n} + x+1}{B_n-\Big(\frac{2}{5} \sqrt{2 \log \frac{2 \Upsilon_n}{y}}  \Upsilon_n^{\frac{1}{2}} +  \log \frac{2 \Upsilon_n}{y}\Big)} \right)^{\frac{1}{2}}.
\end{equation*}
Choose $y = e^{-x}$, then with probability at least $1-2 e^{-x}$,
\begin{equation*}
     |\widehat{\eta}_0 - \eta_0| \leq \frac{R \pi_1}{\frac{1}{2}-\varepsilon}\sqrt{\frac{9 c}{ 2 m_{\pi_1,\varepsilon} \varepsilon}} \left(\frac{\sqrt{(x+1)\Upsilon_n} + x+1}{B_n-\Big(\frac{2}{5} \displaystyle \sqrt{2 (\log 2 \Upsilon_n+x)}  \Upsilon_n^{\frac{1}{2}} +  \log 2 \Upsilon_n + x\Big)} \right)^{\frac{1}{2}}.
\end{equation*}

\begin{lemma}
\label{minoratiosommedesNt2}
Let 
$A_n := \displaystyle \frac{\Upsilon_n(\Upsilon_n-1)(2\Upsilon_n-1)}{96}$. For all $y>0$,
\begin{equation*}
    \PP_{\eta_n}\left( A_n- \sum_{t=1}^{\Upsilon_n} N_{t-1}^2\geq \frac{2}{5} \sqrt{2 \log \frac{2 \Upsilon_n}{y}}  \Upsilon_n^{\frac{5}{2}} + \Upsilon_n^2 \log \frac{2 \Upsilon_n}{y} \right) \leq y. 
\end{equation*}
\end{lemma}

\begin{proof}
For $t \geq 0$, let $M_0 = 0$ and $M_t = \sum_{s=1}^t \big(\one_{I_s = 1}-p_{1,s}^{\eta_n}\big)$ if $t \geq 1$. Then $(M_t)$ is an $(\F_t)$-martingale. Indeed,
\begin{equation*}
    \EE_{\eta_n}\big[M_t|\F_{t-1}\big] = \sum_{s=1}^{t-1} \big(\one_{I_s = 1}-p_{1,s}^{\eta_n}\big)+ \underbrace{\EE_{\eta_n}\big[\one_{I_s = 1}|\F_{t-1}\big] -p_{1,t}^{\eta_n}}_{=0} = M_{t-1}.
\end{equation*}
Using Lemma 3.3 in \cite{Houdre}, for all $\lambda$,
\begin{equation*}
    \mathcal{E}_t := \exp\left(\lambda M_t - \sum_{k \geq 2} \frac{\lambda^k}{k !} \sum_{s=1}^t \EE_{\eta_n} \Big[(M_s-M_{s-1})^k|\F_{s-1}\Big]\right)
\end{equation*}
is an $(\F_t)$-supermartingale. Therefore,
\begin{equation*}
    \EE_{\eta_n}[ \mathcal{E}_t] \leq \EE_{\eta_n}[ \mathcal{E}_1] = 1,
\end{equation*}
that is,
\begin{align}\label{supermartingaleM}
     \EE_{\eta_n} \Big[e^{\lambda M_t}\Big] &\leq  \EE_{\eta_n}\left[\exp \left(\sum_{k \geq 2} \frac{\lambda^k}{k !} \sum_{s=1}^t \EE_{\eta_n} \Big[(M_s-M_{s-1})^k|\F_{s-1}\Big]\right)\right]
     \\& \leq \EE_{\eta_n}\left[\exp \left(\sum_{k \geq 2} \frac{|\lambda|^k}{k !} \sum_{s=1}^t \EE_{\eta_n} \Big[|M_s-M_{s-1}|^k|\F_{s-1}\Big]\right)\right]
\end{align}
using the monotonicity of $\exp$ and the triangle inequality.
The difference $|M_s - M_{s-1}|$ is almost surely bounded:
\begin{equation}\label{majnormeinf}
    |M_s - M_{s-1}| = |\one_{I_s = 1}-p_{1,s}^{\eta_n}| \leq 1,
\end{equation}
and it conditional variance can be controlled by
\begin{equation}\label{majvariance}
     \sum_{s=1}^t \EE_{\eta_n} \Big[(M_s-M_{s-1})^2|\F_{s-1}\Big] =  \sum_{s=1}^t \EE_{\eta_n} \Big[(\one_{I_s = 1}-p_{1,s}^{\eta_n})^2|\F_{s-1}\Big] = \sum_{s=1}^t p_{1,s}^{\eta_n} (1-p_{1,s}^{\eta_n}) \leq \frac{t}{4},
\end{equation}
since $p_{1,s}^{\eta_n}\leq \frac{1}{2}$. Therefore, \eqref{majnormeinf} and \eqref{majvariance} imply
\begin{align}
    \sum_{s=1}^t \EE_{\eta_n} \Big[|M_s-M_{s-1}|^k|\F_{s-1}\Big] &=  \sum_{s=1}^t \EE_{\eta_n} \Big[|M_s-M_{s-1}|^{k-2}(M_s-M_{s-1})^2|\F_{s-1}\Big] \\& \leq \sum_{s=1}^t \EE_{\eta_n} \Big[(M_s-M_{s-1})^2|\F_{s-1}\Big] \leq \frac{t}{4}.
\end{align}
Injecting this equation in \eqref{supermartingaleM} shows that for all $\lambda$,
\begin{equation*}
     \EE_{\eta_n} \Big[e^{\lambda M_t}\Big] \leq \EE_{\eta_n}\left[\exp \left(\sum_{k \geq 2} \frac{|\lambda|^k}{k !} \frac{t}{4}\right)\right] = \exp \left( \lambda^2 \sum_{k \geq 0} \frac{|\lambda|^k}{(k+2) !} \frac{t}{4}\right).
\end{equation*}
Given that $(k+2)! \geq 2^{k+1}$ for all $k \geq 0$, for all $\lambda \in \displaystyle (-2,2)$,
\begin{equation}\label{Markovinequality}
     \EE_{\eta_n} \Big[e^{\lambda M_t}\Big] \leq \exp \left( \frac{\lambda^2}{2} \sum_{k \geq 0} \frac{|\lambda|^k}{2^k} \frac{t}{4}\right) \leq \exp \left( \frac{\lambda^2 t}{8(1- \frac{|\lambda|}{2})}\right).
\end{equation}
Let $x > 0$ and $\lambda \in (0,2)$. Then,
\begin{equation*}
    \PP_{\eta_n} \big( |M_t| \geq x \big) \leq \PP_{\eta_n} \big( M_t \geq x \big) + \PP_{\eta_n} \big( -M_t \geq x \big).
\end{equation*}
Using \eqref{Markovinequality} and Chernoff's bound,
\begin{equation*}
    \PP_{\eta_n} \big( |M_t| \geq x \big) \leq 2 \exp \left( \frac{\lambda^2 t /2}{4(1- \frac{\lambda}{2})}-\lambda x\right).
\end{equation*}
It can be shown that 
\begin{equation*}
    \sup_{\lambda \in [0,2)} \left(\lambda x -\frac{\lambda^2/2 t}{4(1-\lambda/2)} \right) = t  \cdot h(2x/t) \geq \frac{2x^2}{t+2x}
\end{equation*}
where $h(u) = 1+u -\sqrt{1+2u}$ for $u > 0$, using that $h(u) \geq \frac{u^2}{2(1+u)}$ for $u > 0$.
Therefore,
\begin{equation*}
    \PP_{\eta_n} \big( |M_t| \geq x \big) \leq 2 \exp \left( -\frac{2x^2}{t+2x}\right).
\end{equation*}
That is 
\begin{equation}\label{azumahoeffdingNt}
   \PP_{\eta_n} \Big( \big|N_t - \sum_{s=1}^{t} p_{1,s}^{\eta_n}\big| \geq x \Big) \leq 2 \exp \left( -\frac{2x^2}{t+2x}\right)
\end{equation}
For all $t \leq \Upsilon_n$, $0 \leq N_t \leq t$ and $0 \leq \sum_{s=1}^{t} p_{1,s}^{\eta_n} \leq t$, so
\begin{equation*}
    \big|N_t^2 - \big(\sum_{s=1}^{t} p_{1,s}^{\eta_n}\big)^2\big|
    = \big|N_t - \sum_{s=1}^{t} p_{1,s}^{\eta_n}\big| \big(N_t + \sum_{s=1}^{t} p_{1,s}^{\eta_n}\big)
    \leq 2t \big|N_t - \sum_{s=1}^{t} p_{1,s}^{\eta_n}\big|.
\end{equation*}
Therefore,
\begin{align*}
  \PP_{\eta_n} \Big( \big(\sum_{s=1}^{t}p_{1,s}^{\eta_n}\big)^2-x  \geq N_t^2  \Big)
  &\leq \PP_{\eta_n} \Big( \big|N_t^2 - \big(\sum_{s=1}^{t} p_{1,s}^{\eta_n}\big)^2\big| \geq x \Big) \\
  &\leq \PP_{\eta_n} \Big( 2t \big|N_t - \sum_{s=1}^{t} p_{1,s}^{\eta_n}\big| \geq x \Big)
  \leq 2 \exp \left( -\frac{x^2}{2 (t^3+x t)}\right). 
\end{align*}

\begin{lemma}[Inequality Reversal (Lemma 1 in \cite{Peel})]\label{reversallemma}
Let $X$ be a random variable and $a$, $b >0$, $c$, $d \geq 0$ such that 
\begin{align*}
    \forall x > 0, \PP (|X| \geq x) \leq a \exp \left( -\frac{b x^2}{c+d x}\right).
\end{align*}
Then, 
\begin{equation*}
    \forall y > 0, \PP \left(|X| \geq \sqrt{\frac{c}{b}\log \frac{a}{y}}+\frac{d}{b}\log \frac{a}{y}\right) \leq y.
\end{equation*}
\end{lemma}
Using this Lemma, the previous equation implies: for all $u > 0$,
\begin{equation*}
  \PP_{\eta_n} \left( \big(\sum_{s=1}^{t}p_{1,s}^{\eta_n}\big)^2- N_t^2  \geq H(t,u)\right) \leq u, 
\end{equation*}
where $H(t,u) :=\displaystyle t^{\frac{3}{2}} \sqrt{2\log \frac{2}{u}}  + 2 t \log \frac{2}{u}$.
Using Lemma~\ref{sommerlesproba} and summing the probabilities from $t =1$ to $\Upsilon_n$,
\begin{equation}\label{inegaliteconcentrationsommedescarresNt1}
    \PP_{\eta_n} \Big( \sum_{t=1}^{\Upsilon_n} \big(\sum_{s=1}^{t-1}p_{1,s}^{\eta_n}\big)^2- \sum_{t=1}^{\Upsilon_n} N_{t-1}^2\geq \sum_{t=1}^{\Upsilon_n} H(t-1,u)   \Big) \leq \Upsilon_n u. 
\end{equation}
Now, 
\begin{align*}
    \sum_{t=1}^{\Upsilon_n} H(t-1,u) &\leq \sqrt{2 \log \frac{2}{u}} \int_{0}^{\Upsilon_n} t^{\frac{3}{2}} dt + \Upsilon_n (\Upsilon_n -1) \log \frac{2}{u}
    \\& \leq \frac{2}{5} \sqrt{2 \log \frac{2}{u}}  \Upsilon_n^{\frac{5}{2}} + \Upsilon_n^2 \log \frac{2}{u}.
\end{align*}
Choosing $y = \Upsilon_n u$, \eqref{inegaliteconcentrationsommedescarresNt1} becomes
\begin{equation}\label{inegaliteconcentrationsommedescarresNt}
    \PP_{\eta_n} \Big( \sum_{t=1}^{\Upsilon_n} \big(\sum_{s=1}^{t-1}p_{1,s}^{\eta_n}\big)^2- \sum_{t=1}^{\Upsilon_n} N_{t-1}^2\geq \frac{2}{5} \sqrt{2 \log \frac{2 \Upsilon_n}{y}}  \Upsilon_n^{\frac{5}{2}} + \Upsilon_n^2 \log \frac{2 \Upsilon_n}{y}   \Big) \leq y. 
\end{equation}
From the proof of Proposition \ref{choiceofTau} (see Section \ref{proofofpropositionchoiceofTau}), for $s \leq \Upsilon_n$, $p_{1,s}^{\eta_n} \geq \frac{1}{2}-\eta_n \pi_1 s \geq \frac{1}{2}-\frac{R}{n^{\alpha}} s $.
Therefore,
\begin{equation*}
    \sum_{s=1}^{t-1}p_{1,s}^{\eta_n} \geq \frac{1}{2}(t-1) - \frac{R}{n^{\alpha}}  \frac{t(t-1)}{2}  = \frac{t-1}{2} \left(1- \frac{R}{n^{\alpha}} t\right),
\end{equation*}
and since $t \leq \displaystyle \Upsilon_n = \Big(\frac{1}{2}-\varepsilon\Big) \frac{n^{\alpha}}{R}$,
\begin{equation*}
\sum_{s=1}^{t-1}p_{1,s}^{\eta_n}
    \geq \frac{t-1}{2} \left(1 - (\frac{1}{2}-\varepsilon)\frac{t}{\Upsilon_n}\right)
    \geq \frac{t-1}{4},
\end{equation*}
so that
\begin{align*}
     \sum_{t=1}^{\Upsilon_n} \big(\sum_{s=1}^{t-1}p_{1,s}^{\eta_n}\big)^2
    \geq \frac{1}{16} \sum_{t=1}^{\Upsilon_n} (t-1)^2
    =: A_n
    = \frac{\Upsilon_n^3}{48} + O(\Upsilon_n^2).
\end{align*}
Injecting this equation in Equation~\eqref{inegaliteconcentrationsommedescarresNt},
\begin{align*}
    \PP_{\eta_n} &  \Big( A_n- \sum_{t=1}^{\Upsilon_n} N_{t-1}^2\geq \frac{2}{5} \sqrt{2 \log \frac{2 \Upsilon_n}{y}}  \Upsilon_n^{\frac{5}{2}} + \Upsilon_n^2 \log \frac{2 \Upsilon_n}{y} \Big)\\& \leq
    \PP_{\eta_n} \Big( \sum_{t=1}^{\Upsilon_n} \big(\sum_{s=1}^{t-1}p_{1,s}^{\eta_n}\big)^2-\sum_{t=1}^{\Upsilon_n} N_{t-1}^2 \geq \frac{2}{5} \sqrt{2 \log \frac{2 \Upsilon_n}{y}}  \Upsilon_n^{\frac{5}{2}} + \Upsilon_n^2 \log \frac{2 \Upsilon_n}{y}   \Big) \leq y . 
\end{align*}
\end{proof}

\section{NUMERICAL ILLUSTRATIONS}

The numerical illustrations have been realized using \texttt{R}. All the material necessary for the reproduction of the simulations of the article is contained in the same file as the supplementary in a zip file. Let's explain the link between the different data files and the available algorithms.

\begin{itemize}
    \item All necessary functions are in the file \texttt{myfunctions.R}. It contains in particular the procedure \texttt{Exp3}, but also the functions which calculate the MLE for a constant learning rate and the truncated MLE for a decreasing learning rate.
    \item The algorithm \texttt{est\_etaconstant\_nveclarge.R} generates the data file \texttt{Data\_eta03\_nveclarge.Rdata} containing the estimators for a constant learning rate necessary for Figure \ref{est_etacst}.
    \item The algorithm \texttt{Tmax\_comp.R} generates the data file \texttt{Tmax\_eta03\_dec\_K4.Rdata} containing the $\Upsilon_{\max}$ necessary for Figure \ref{evolutionofTmax}.
    \item The algorithm \texttt{est\_etadecreasing.R} generates the data file \texttt{Data\_eta03\_dec.Rdata} or  \texttt{Data\_eta03\_dec\_K4.Rdata} (depending on the number of arms) containing the estimators for a decreasing learning rate respectively for $2$ arms and $4$ arms. These data are used in Figures \ref{pred} and \ref{est}.
    \item The Markdown file \texttt{Test\_Exp3.Rmd} reproduces all the figures from the article for these different data files. To spare the reader lengthy computations, the output of the functions \texttt{est\_etaconstant\_nveclarge.R}, \texttt{Tmax\_comp.R} and \texttt{est\_etadecreasing.R}, is made available in the previous enumerated data files.
\end{itemize}


\end{document}